\documentclass{article}

\usepackage[final]{neurips_2019}

\usepackage[utf8]{inputenc} % allow utf-8 input
\usepackage[T1]{fontenc}    % use 8-bit T1 fonts
\usepackage{hyperref}       % hyperlinks
\usepackage{url}            % simple URL typesetting
\usepackage{booktabs}       % professional-quality tables
\usepackage{amsfonts}       % blackboard math symbols
\usepackage{nicefrac}       % compact symbols for 1/2, etc.
\usepackage{microtype}      % microtypography
\usepackage{algorithm}
\usepackage{algorithmic}
\usepackage{amsmath}
\usepackage{amssymb}
\usepackage{amsthm}
\usepackage{bm}
\usepackage{caption}
\usepackage{graphicx}
\usepackage{wrapfig}
\usepackage{authblk}

\newtheorem{theorem}{Theorem}
\newtheorem{lemma}{Lemma}
\newtheorem{proposition}{Proposition}

\theoremstyle{definition}
\newtheorem{assumption}{Assumption}

\title{Value Propagation for Decentralized Networked Deep Multi-agent  Reinforcement Learning}

\author[      1]{Chao Qu \thanks{luoji.qc@antfin.com}}
\author[2]{Shie Mannor}
\author[3,4]{Huan Xu}
\author[1]{Yuan Qi}
\author[1,4]{Le Song}
\author[1]{Junwu Xiong}
\affil{Ant Financial Services Group}\affil[2]{ Technion}\affil[3]{Alibaba Group}\affil[4]{Georgia Institute of Technology}
\begin{document}

\maketitle

\begin{abstract}
 We consider the networked multi-agent reinforcement learning (MARL) problem in a fully decentralized setting, where agents learn to coordinate to achieve joint success. This problem is widely encountered in many areas including  traffic control, distributed control, and smart grids. 
 We assume each agent is located at a node of a communication network and can  exchange information only with its neighbors.  Using  softmax temporal consistency,  we derive a primal-dual decentralized optimization method and obtain a principled and data-efficient iterative algorithm named {\em value propagation}.  We prove a non-asymptotic convergence rate of $\mathcal{O}(1/T)$ with  nonlinear function approximation. To the best of our knowledge, it is the first MARL algorithm with a convergence guarantee in the control, off-policy, non-linear function approximation, fully decentralized setting. 
\end{abstract}

\section{Introduction}

\setlength{\abovedisplayskip}{4pt}
\setlength{\abovedisplayshortskip}{1pt}
\setlength{\belowdisplayskip}{4pt}
\setlength{\belowdisplayshortskip}{1pt}
\setlength{\jot}{3pt}
\setlength{\textfloatsep}{6pt}	

Multi-agent systems have applications in a wide range of areas  such as robotics, traffic control, distributed control, telecommunications, and economics. For these areas, it is often difficult or simply impossible to predefine agents' behaviour to achieve satisfactory results, and  
multi-agent reinforcement learning (MARL) naturally arises~\citep{bu2008comprehensive,tan1993multi}. 
For example,  \citet{el2013multiagent} model a traffic signal control problem as a multi-player stochastic game and solve it with MARL.
MARL generalizes reinforcement learning by considering a set of   agents (decision makers) sharing a common environment.  However, multi-agent reinforcement learning is a challenging problem since the agents interact with both the environment and each other. For instance, independent Q-learning---treating other agents as a part of the environment---often fails as the multi-agent setting breaks the theoretical convergence guarantee of Q-learning and makes the learning process unstable \citep{tan1993multi}. \citet{rashid2018qmix,foerster2018counterfactual,lowe2017multi} alleviate such a problem using a centralized network (i.e., being centralized for training, but decentralized during execution.). Its communication pattern is illustrated in the left panel of Figure \ref{fig:network}.

Despite the great success of (partially) centralized MARL approaches, there are various scenarios, such as sensor networks \citep{rabbat2004distributed} and intelligent transportation systems \citep{adler2002cooperative} , where a central agent does not exist or may be too expensive to use. In addition, \textit{privacy} and \textit{security} are requirements of many real world problems in multi-agent system (also in many modern machine learning problems) \citep{abadi2016deep,kurakin2016adversarial}  . For instance, in Federated learning \citep{mcmahan2016communication}, the learning task is solved by a lose federation of participating devices (agents) without the need to centrally store the data, which significantly reduces privacy and security risk by limiting the attack surface to only the device. In the agreement problem \citep{degroot1974reaching,mo2017privacy}, a group of agents may want to reach consensus on a subject without leaking their \textit{individual goal} or \textit{opinion} to others. Obviously, centralized MARL violates privacy and security requirements. To this end, we and others have advocated the \textit{fully decentralized} approaches, which are useful for many applications including unmanned
vehicles \citep{fax2002information}, power grid \citep{callaway2011achieving}, and sensor networks \citep{cortes2004coverage}.  
For these approaches, we can use a network to model the interactions between agents (see the right panel of Figure \ref{fig:network}). Particularly, We consider a \emph{fully cooperative} setting where each agent makes its own decision based on  its \textit{local reward} and messages received from their neighbors.  Thus each agent preserves the \textit{privacy} of its own \textit{goal} and \textit{policy}. At the same time, through the message-passing all agents achieve consensus to maximize the averaged cumulative rewards over all agents; see Equation \eqref{equ:objective}.  

\begin{wrapfigure}{r}{0.5\textwidth}
	\begin{center}
		\includegraphics[width=0.48\textwidth]{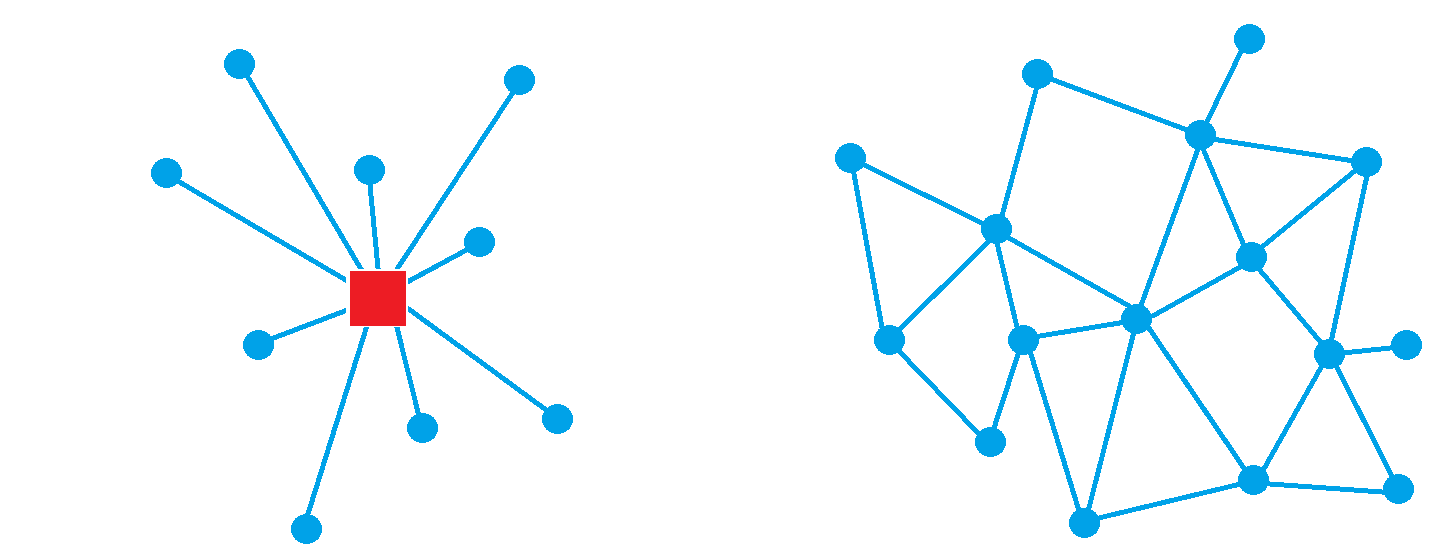}
	\end{center}
	\caption{Centralized network vs Decentralized network. Each blue node in the figure corresponds to an agent. In centralized network (left), the red central node collects information for all agents, while in decentralized network (right), agents exchanges information with neighbors.}
	\label{fig:network}
\end{wrapfigure}
In this paper, we propose a new fully decentralized networked multi-agent deep reinforcement learning algorithm. 
Using  softmax temporal consistency \citep{nachum2017bridging,dai2018sbeed}  
to connect value and policy updates, 
 we derive a new two-step primal-dual decentralized reinforcement learning algorithm inspired by a primal decentralized optimization method \citep{hong2017prox} \footnote{The objective in \cite{hong2017prox} is a primal optimization problem with constraint. Thus they introduce a Lagrange multiplier like method to solve it (so they call it primal-dual method ). Our objective function is  a primal-dual optimization problem with constraint.  }.
In the first step of each iteration, each agent computes its local policy, value gradients and dual gradients and then updates only policy parameters. In the second step, each  agent propagates  to its neighbors the messages based on its value function (and dual function) and then updates its own value function. Hence we name the algorithm \emph{value propagation}.  It preserves the privacy in the sense that no individual reward function is required for the network-wide collaboration. 
 We approximate the local policy, value function and dual function of each agent by deep neural networks, which enables automatic feature generation and end-to-end learning.

\textbf{Contributions:}
\textbf{[1]} We propose the value propagation algorithm and prove that it converges with the rate $\mathcal{O} (1/T)$ even with the \emph{non-linear} deep neural network function approximation.  To the best of our knowledge, it is the first deep MARL algorithm with non-asymptotic  convergence guarantee. At the same time, value propagation can use off-policy updates, making it data efficient. When it reduces to the single agent case, it provides a  proof of \citep{dai2018sbeed} in the realistic setting; see remarks of algorithm \ref{alg:VP} in Section \ref{section:algorithm_VP}.
\textbf{[2]} The objective function in our problem is a primal-dual decentralized optimization form (see \eqref{equ:final_objective_VP}), while the objective function in \citep{hong2017prox} is a primal problem. When our method reduces to pure primal analysis, we extend \citep{hong2017prox} to the stochastic and biased gradient setting which may be of independent interest to the optimization community. In the practical implementation, we extend ADAM into the decentralized setting to accelerate training.

\section{Preliminaries}

\textbf{MDP}~ Markov Decision Process (MDP) can be described by  a 5-tuple ($\mathcal{S}, \mathcal{A}, \mathcal{R}, \mathcal{P}, \gamma$): $\mathcal{S}$ is the finite state space, $\mathcal{A}$ is the finite action space, $\mathcal{P}=(P(s'|s,a))_{s,s'\in \mathcal{S},a\in \mathcal{A}}$ are the transition probabilities, $R=(R(s,a))_{s,s'\in \mathcal{S},a\in \mathcal{A}}$ are the real-valued immediate rewards and $\gamma\in (0,1)$ is the discount factor. A policy is used to select actions in the MDP. In general, the policy is stochastic and denoted by $\pi$, where $\pi(s_t,a_t)$ is the conditional probability density at $a_t$ associated with the policy. Define $V^*(s)=\max_{\pi} \mathbb{E} [\sum_{t=0}^{\infty} \gamma^t R(s_t,a_t)| s_0=s]$ to be the optimal value function. It is known that $V^*$ is the unique fixed point of the Bellman optimality operator,
$V(s)=(\mathcal{T}V)(s):=\max_a R(s,a)+\gamma \mathbb{E}_{s'|s,a}[V(s')]. $
The optimal policy $\pi^*$ is related to $V^*$ by the following equation: $\pi^{*} (s,a)=\arg\max_{a} \{ R(s,a)+\gamma \mathbb{E}_{s'|s,a} V^{*}(s')  \} $

\textbf{Softmax Temporal Consistency}~ \citet{nachum2017bridging} establish a connection between value and policy based reinforcement learning based on a relationship between softmax temporal value consistency and policy optimality under entropy regularization. Particularly, the soft Bellman optimality is as follows,
\begin{equation}\label{equ:soft_bellman}
V_{\lambda} (s)=
\max_{\pi(s,\cdot)} \big( \mathbb{E}_{a\sim \pi(s,\cdot)} (R(s,a)+\\
\gamma \mathbb{E}_{s'|s,a} V_{\lambda}(s') )+\lambda H(\pi, s)  \big),
\end{equation}
where $H(\pi,s)=-\sum_{a\in \mathcal{A}}\pi(s,a)\log\pi(s,a)$ and $\lambda\geq 0$ controls the degree of regularization. When $\lambda=0$, above equation reduces to the standard Bellman optimality condition. An important property of  soft Bellman optimality  is the called temporal consistency, which leads to the Path Consistency Learning.

% It is not hard to derive that 
%\begin{equation}
%\begin{split}
%V_\lambda&=(\mathcal{T} V_{\lambda}) (s)\\
%&:=\lambda \log\big( \sum_{a\in \mathcal{A}} \exp(\frac{R(s,a)+\gamma \mathbb{E}_{s'|s,a}V_{\lambda} (s') }{\lambda})\big),
%\end{split}
%\end{equation}
%which is a softmax function over $ R(s,a)+\gamma \mathbb{E}_{s'|s,a}V_{\lambda} (s') $. 

\begin{proposition} {\citep{nachum2017bridging}\label{prop:temporal_consistency}.}
	Assume $\lambda>0$. Let $V^{*}_{\lambda}$ be the fixed point of \eqref{equ:soft_bellman} and $\pi^{*}_\lambda$ be the corresponding policy that attains that maximum on the RHS of \eqref{equ:soft_bellman}. Then, $(V^*_{\lambda},\pi_\lambda^*)$ is the unique $(V,\pi)$ pair that satisfies the following equation for all $(s,a)\in \mathcal{S}\times \mathcal{A}:$	
	$ V(s)=R(s,a)+\gamma\mathbb{E}_{s'|s,a} V(s')-\lambda \log \pi(s,a). $
\end{proposition}

A straightforward way to apply temporal consistency is to optimize the following problem, $\min_{V, \pi} E_{s,a}\big( R(s,a)+\gamma \mathbb{E}_{s'|s,a} V(s')-\lambda \log\pi(s,a)-V(s) \big)^2.$
\citet{dai2018sbeed} get around the double sampling problem of above formulation by introduce a primal-dual form
\begin{equation}
\min_{V,\pi}\max_{\rho}\mathbb{E}_{s,a,s'} [ (\delta(s,a,s')-V(s))^2]-\eta \mathbb{E}_{s,a,s'}[(\delta (s,a,s')-\rho(s,a))^2],
\end{equation}
where $\delta(s,a,s')=R(s,a)+\gamma V(s')-\lambda \log \pi(s,a)$, $0\leq \eta\leq 1$ controls the trade-off between bias and variance.

In the following discussion, we use  $\|\cdot\|$ to denote the Euclidean norm over the vector, $A'$ stands for the transpose of $A$, and $\odot$ denotes the entry-wise product between two vectors.

\section{Value Propagation }\label{section:VP}

In this section, we present our multi-agent reinforcement learning algorithm, i.e., value propagation. To begin with, we extend the MDP model to the Networked Multi-agent MDP model following the definition in \citep{zhang2018fully}.  Let $\mathcal{G}=(\mathcal{N},\mathcal{E})$ be an undirected graph with $|\mathcal{N}|=N$ agents (node). $\mathcal{E}$ represents the set of edges. $(i,j)\in \mathcal{E}$ means agent $i$ and $j$ can communicate with each other through this edge.  A networked multi-agent MDP is characterized by a tuple $(\mathcal{S}, \{\mathcal{A}^i\}_{i\in \mathcal{N}}, \mathcal{P}, \{R^i\}_{i\in \mathcal{N}}, \mathcal{G}, \gamma)$: $\mathcal{S}$ is the global state space shared by all agents (It could be partially observed, i.e., each agent observes its own state $S^i$, see our experiment). $\mathcal{A}^i$ is the action space of agent $i$,  $\mathcal{A}=\prod_{i=1}^{N}\mathcal{A}^i$ is the joint action space, $\mathcal{P}$ is the transition probability, $\mathcal{R}^i$ denotes the \textit{local} reward function of agent $i$. We assume rewards are observed only locally to preserve the privacy of the each agent's goal. At each time step, agents observe $s_t$ and make the decision $a^t=(a_1^t, a_2^t,..., a_N^t)$. Then each agent just receives its own reward $R_i(s_t,a_t)$,  and the environment  switches to the new state $s_{t+1}$ according to the transition probability. Furthermore, since each agent make the decisions independently, it is reasonable to assume that the policy $\pi(s,a)$ can be factorized, i.e.,  $\pi (s,a)=\prod_{i=1}^{N} \pi^{i}(s,a^i)$ \citep{zhang2018fully}. We call our method \textit{fully-decentralized} method, since  reward is received locally, the action is executed locally by agent, critic (value function) are trained locally.
\subsection{Multi-Agent Softmax Temporal Consistency}

The goal of the agents is to learn a policy that maximizes the long-term reward averaged over the agent, i.e.,
\begin{equation}\label{equ:objective}
\mathbb{E}\sum_{t=0}^{\infty} \frac{1}{N}\sum_{i=1}^{N}  \gamma^t R_i(s_t,a_t).
\end{equation}
In the following, we adapt the temporal consistency into the multi-agent version. Let $V_{\lambda}(s)= \mathbb{E}\big(\frac{1}{N}\sum_{i=1}^{N} R_i(s,a)+\gamma \mathbb{E}_{s'|s,a} V_{\lambda} (s') +\lambda H(\pi,s)\big),$ $V^*_{\lambda}$ be the optimal value function and $\pi_\lambda^*(s,a)=\prod_{i=1}^{N} \pi_\lambda^{i*} (s,a^i)$ be the corresponding policy. Apply the soft temporal consistency, we obtain that for all $(s,a)\in \mathcal{S}\times \mathcal{A}$, $(V^*_\lambda,\pi^*_\lambda)$ is the unique $(V,\pi)$ pair that satisfies
\begin{equation}\label{equ:multi_temporal_consistency}
\begin{split}
V(s)=\frac{1}{N}\sum_{i=1}^{N} R_i(s,a) +\gamma \mathbb{E}_{s'|s,a} V(s')-\lambda \sum_{i=1}^{N}\log \pi^i(s,a^i).
\end{split}
\end{equation}
A  optimization problem inspired by \eqref{equ:multi_temporal_consistency} would be 
\begin{equation}\label{equ:multi_PCL_primal}
\min_{ \{\pi^i\}_{i=1}^{N}, V}\mathbb{E} \big(  V(s)-  \frac{1}{N} \sum_{i=1}^{N}R_i(s,a)-\gamma \mathbb{E}_{s'|s,a}  V (s')
+\lambda \sum_{i=1}^{N} \log \pi^i (s,a^i) \big)^2.   
\end{equation}

There are two potential issues of above formulation:
First, due to the inner conditional expectation, it would require two independent samples to obtain the unbiased estimation of gradient of $V$ \citep{dann2014policy}.
Second, $V (s)$ is a global variable over the network, thus can not be updated in a \textit{decentralized} way. 

For the first issue, we introduce the primal-dual form of \eqref{equ:multi_PCL_primal} as that in \citep{dai2018sbeed}. Using the fact that $x^2=\max_\nu (2\nu x- \nu^2)$ and the interchangeability principle \citep{shapiro2009lectures} we have,
\begin{equation*}
\min_{ V,  \{\pi^i\}_{i=1}^{N} }\max_{\nu} 2\mathbb{E}_{s,a,s'} [\nu(s,a) \big( \frac{1}{N}\sum_{i=1}^{N} ( R_i(s,a)+ \gamma V(s') -V(s)
-\lambda N\log\pi^i(s,a^i) \big)]-\mathbb{E}_{s,a,s}[\nu^2(s,a)].
\end{equation*}

Change the variable $\nu(s,a)=\rho(s,a)-V(s)$, the objective function becomes
\begin{equation}\label{equ:multi_objective_central_first}
\min_{V, \{  \pi^i\}_{i=1}^{N}}\max_{\rho} \mathbb{E}_{s,a,s'}[\big(\frac{1}{N} \sum_{i=1}^{N} (\delta_i(s,a,s')-V(s)) \big)^2] 
-\mathbb{E}_{s,a,s'}[ \big( \frac{1}{N}\sum_{i=1}^{N} (\delta_i(s,a,s')-\rho(s,a))\big)^2  ],  
\end{equation}
where $\delta_i= R_i(s,a)+\gamma V(s')-\lambda N \log\pi^i(s,a^i).$
\subsection{Decentralized Formulation} 

So far the problem is still in a centralized form, and we now turn to  reformulating it in a decentralized way. We assume that policy, value function, dual variable $\rho$ are all in the parametric function class. Particularly, each agent's policy is $\pi^i(s,a^i):=\pi_{\theta
	_{\pi^i}}(s,a^i)$ and $\pi_{\theta}(s,a) = \prod_{i=1}^{N} \pi_{ \theta_{\pi^i}}(s,a^i).$ The value function  $ V_{\theta_v}(s)$ is characterized by the parameter $\theta_v$, while   $\theta_{\rho}$ represents the parameter of $\rho(s,a)$. Similar to \citep{dai2018sbeed}, we optimize a slightly different version from \eqref{equ:multi_objective_central_first}.
\begin{align}\label{equ:multi_objective_central}
\min_{\theta_v, \{  \theta_{\pi^i} \}_{i=1}^{N}}\max_{\theta_\rho} \mathbb{E}_{s,a,s'}[\big(\frac{1}{N} \sum_{i=1}^{N} (\delta_i(s,a,s')-V(s)) \big)^2]-\eta\mathbb{E}_{s,a,s'}[ \big( \frac{1}{N}\sum_{i=1}^{N} (\delta_i(s,a,s')-\rho(s,a))\big)^2  ],  
\end{align}
where $0\leq\eta\leq 1$ controls the bias and variance  trade-off. When $\eta=0$, it reduces to the pure primal form. 

We now consider the second issue that $V(s)$ is a global variable. To address this problem, we introduce the local copy of the value function, i.e., $V_i(s)$ for each agent $i$. In the algorithm, we have a consensus update step, such that these local copies are the same, i.e., $V_1(s)=V_2(s)=...=V_N(s)= V(s)$, or equivalently  $ \theta_{v^1}=\theta_{v^2}=...=\theta_{v^N}$, where $\theta_{v^i}$ are parameter of $V_i$ respectively. Notice now in \eqref{equ:multi_objective_central}, there is a global dual variable $\rho$ in the primal-dual form. Therefore, we also introduce the local copy of the dual variable, i.e.,  $\rho_i(s,a)$ to formulate it into the decentralized optimization problem.  Now the \textit{final} objective function we need to optimize is
\begin{align}\label{equ:final_objective_VP}
\min_{ \{ \theta_{v^i}, \theta_{\pi^i} \}_{i=1}^{N} }&\max_{\{\theta_{\rho^i} \}_{i=1}^{N}} L(\theta_V,\theta_\pi,\theta_\rho)= \mathbb{E}_{s,a,s'}[\big(\frac{1}{N}\sum_{i=1}^{N} (\delta_i(s,a,s')-V_{i}(s)) \big)^2] \nonumber \\
&\hspace{5cm}-\eta\mathbb{E}_{s,a,s'}[ \big( \frac{1}{N}\sum_{i=1}^{N} (\delta_i(s,a,s')-\rho_{i}(s,a))\big)^2  ], \nonumber \\
&s.t. ~~ \theta_{v^1}=,...,=\theta_{v^N}, \theta_{\rho^1}=,...,=\theta_{\rho^N},  
\end{align}
where $\delta_i= R_i(s,a)+\gamma V_{i}(s')-\lambda N \log\pi^i(s,a^i).$  We are now ready to present the value propagation algorithm. In the following, for notational simplicity, we assume the parameter of each agent is a scalar, i.e., $\theta_{\rho^i}, \theta_{\pi^i}, \theta_{v^i} \in R$.  We pack the parameter together and slightly abuse the notation by writing $\theta_\rho=[\theta_{\rho^1},...,\theta_{\rho^N}]'$, $\theta_\pi=[\theta_{\pi^1},..., \theta_{\pi^N}]'$, $\theta_{V}=[\theta_{v^1},...,\theta_{v^N}]'.$ Similarly, we also pack the stochastic gradient $ g(\theta_{\rho})=[ g(\theta_{\rho^1}),...,g(\theta_{\rho^n}) ]' $, $ g(\theta_{V})=[ g(\theta_{v^1}),...,g(\theta_{v^n}) ]' $.

\subsection{Value propagation algorithm}\label{section:algorithm_VP}

 Solving \eqref{equ:final_objective_VP} even without constraints is not an easy problem when both primal and dual parts are approximated by the deep neural networks.   An ideal way is to optimize the inner dual problem and find the solution 
$ \theta_\rho^*=\arg\max_{\theta_\rho} L(\theta_V, \theta_{\pi}, \theta_\rho)$, such that $\theta_{\rho^1}=...=\theta_{\rho^N} $. Then we can do the (decentralized) stochastic gradient decent to solve the primal problem.
\begin{equation}
\min_{\{\theta_{v^i},\theta_{\pi^i}\}_{i=1}^{N}} L(\theta_V, \theta_{\pi}, \theta_{\rho}^*) \\
~~s.t.~~ \theta_{v^1}=...=\theta_{v^N}.
\end{equation}
However in practice, one \textit{tricky issue} is that  we can not get  the exact solution $\theta^*_\rho$ of the dual problem. Thus, we do the (decentralized) stochastic gradient for $T_{dual}$ steps in the dual problem and get an approximated solution $\tilde{\theta}_\rho$ in the Algorithm \ref{alg:VP}. In our analysis, we take the error $\varepsilon$ generated from this inexact solution into the consideration and analyze its effect on the convergence. Particularly, since $ \nabla_{\theta_{V}} L(\theta_V,\theta_\pi,\tilde{\theta}_{\rho}) \neq\nabla_{\theta_{V}} L(\theta_V,\theta_\pi,\theta^*_{\rho})$, the primal gradient is biased and the results in \citep{dai2018sbeed,hong2017prox} do not fit this problem.

 In the  dual update  we do a consensus update $\theta^{t+1}_{\rho}=\frac{1}{2} D^{-1}L^{+} \theta^t_{\rho}-\frac{\alpha_\rho}{2} D^{-1} A' \mu^t_{\rho}+\frac{\alpha_\rho}{2}D^{-1} g(\theta^t_{\rho})$ using the stochastic gradient of each agent, where $
\mu_\rho$ is some auxiliary variable to incorporate the communication, $D$ is the degree matrix, $A$ is the node-edge incidence matrix, $L^{+}$ is  sign-less graph Laplacian. We defer the detail definition and the derivation of this algorithm to  Appendix \ref{appendix:topology} and Appendix \ref{appendix:update_rule} due to  space limitation.  After  updating the dual parameters, we optimize the primal parameters $\theta_v$, $\theta_{\pi}$. Similarly, we use a mini-batch data from the replay buffer and then do a consensus update on $\theta_v$. The same remarks on $\rho$ also hold for the primal parameter $\theta_v$. Notice here we do not need the consensus update on $\theta_\pi$, since each agent's policy $\pi^i(s,a^i)$ is different than each other.    This update rule is adapted from a primal decentralized optimization algorithm \citep{hong2017prox}. Notice even in the pure primal case, \citet{hong2017prox} only consider the batch gradient case while our algorithm and analysis include the stochastic and biased gradient case. In practicals implementation, we consider the decentralized momentum method and multi-step temporal consistency to accelerate the training; see details in Appendix \ref{appendix:acceleration} and Appendix \ref{appendix:multi-step}.

\textbf{Remarks on  Algorithm \ref{alg:VP}}. 
\textbf{(1)} In the \textit{single} agent case, \cite{dai2018sbeed} assume the dual problem can be exactly solved and thus they analyze a simple pure primal problem. However such assumption is unrealistic especially when the dual variable is represented by the deep neural network. Our \textit{multi-agent} analysis considers the \textit{inexact} solution. This is much harder than that in \citep{dai2018sbeed}, since now the primal gradient  is biased. \textbf{(2)} The update  of each agent just needs the information of the agent itself and its neighbors. See this from the definition of $D$, $A$, $L^{+}$ in the appendix. \textbf{(3)} The topology of the Graph $\mathcal{G}$ affects the convergence speed. In particular, the rate depends on $\sigma_{\min}(A'A)$ and $\sigma_{\min}(D)$, which are related to spectral gap of the network.

\begin{algorithm}[h]
	\caption{Value Propagation}
	\label{alg:VP}
	\begin{algorithmic}
		\STATE { Input: Environment ENV, learning rate $\alpha_\pi$, $\alpha_v$, $ \alpha_\rho$, discount factor $\gamma$, number of step $T_{dual}$ to train dual parameter $\theta_{\rho^i}$,   replay buffer capacity $B$,  node-edge incidence matrix $A\in R^{E\times N}$, degree matrix $D$, signless graph Laplacian $L^{+}$}.	
		\STATE {Initialization of $\theta_{v^i}, \theta_{\pi^i}, \theta_{\rho^i}$, $\mu_\rho^0=0$, $\mu_V^0=0$}. 
		\FOR{$t=1,...,T$}
		\STATE{ sample trajectory $s_{0:\tau}\sim \pi(s,a)=\prod_{i=1}^{N}\pi^i(s,a^i) $ and add it into the replay buffer.}\\		
		\textbf{1. Update the dual parameter} $\theta_{\rho^i}$ \\
		
		Do following dual update $T_{dual}$ times:\\
		
		\STATE{Random sample a mini-batch of transition $(s_t, \{a^i_t\}_{i=1}^{N},s_{t+1}, \{r^i_t\}_{i=1}^N)$ from the replay buffer.}
		
		\FOR {agent $i=1$  to $n$}
		\STATE{Calculate the stochastic gradient $g(\theta^t_{\rho^i})$ of $-\eta(\delta_i(s_t,a_t,s_{t+1})-\rho_i(s_t,a_t))^2$ w.r.t. $\theta^t_{\rho^i}$.  }
		\ENDFOR \\
		// Do consensus update on  $\theta_\rho:=[\theta_{\rho^1},...,\theta_{\rho^N}]'$   \\
		$\theta^{t+1}_{\rho}=\frac{1}{2} D^{-1}L^{+} \theta^t_{\rho}-\frac{\alpha_\rho}{2} D^{-1} A' \mu^t_{\rho}+\frac{\alpha_\rho}{2}D^{-1} g(\theta^t_{\rho}), \mu^{t+1}_\rho=\mu^t_\rho+\frac{1}{\alpha_\rho} A\theta_{\rho}^{t+1}$ \\
		\textbf{2. Update primal parameters} $\theta_{v^i}, \theta_{\pi^i}$
		\STATE{Random sample a mini-batch of transition $(s_t, \{a^i_t\}_{i=1}^{N},s_{t+1}, \{r^i_t\}_{i=1}^N)$ from the replay buffer.}
		\FOR{ agent $i=1$ to $n$ }
		\STATE{Calculate the stochastic gradient $g(\theta^t_{v^i})$,$g(\theta^t_{\pi^i})$ of $(\delta_i(s_t,a_t,s_{t+1})-V_i(s_t))^2-\eta (\delta_i(s_t,a_t,s_{t+1})-\rho_i(s_t,a_t))^2 $,  w.r.t. $\theta^t_{v^i}$, $\theta^t_{\pi^i}$}
		\ENDFOR\\	
		// Do gradient decent on $ \theta_{\pi^i}$:
		$ \theta^{t+1}_{\pi^i}=\theta^{t}_{\pi^i}-\alpha_{\pi} g(\theta^t_{\pi^i})$ for each agent $i$.\\
		
		// Do consensus update on $\theta_{V}:=[\theta_{v^1},...,\theta_{v^N}]'$ :\\
		$\theta^{t+1}_{V}=\frac{1}{2} D^{-1}L^{+} \theta^t_{V}-\frac{\alpha_v}{2} D^{-1} A' \mu^t_{V}-\frac{\alpha_v}{2}D^{-1} g(\theta^t_{V}), \mu^{t+1}_{V}=\mu^t_{V}+\frac{1}{\alpha_v} A\theta_V^{t+1}.$ 
		\ENDFOR
	\end{algorithmic}
\end{algorithm}

\section{Theoretical Result}

In this section, we give the convergence result on Algorithm \ref{alg:VP}.  We first make two mild assumptions on the function approximators $f(\theta)$ of $V_i(s)$, $\pi^i(s,a^i)$, $\rho_i(s,a)$. 

\begin{assumption}\label{assumption}
	
	i) The function approximator $f(\theta)$ is differentiable and has Lipschitz continuous gradient, i.e., $\|\nabla f(\theta_1)-\nabla f(\theta_2)\|\leq L\|\theta_1-\theta_2\|, \forall \theta_1, \theta_2 \in R^K.$
	This is commonly assumed in the non-convex optimization.	ii) The function approximator $f(\theta)$ is lower bounded. This can be easily satisfied when the parameter is bounded, i.e., $\|\theta\|\leq C$ for some positive constant $C$. 
\end{assumption}

In the following, we give the theoretical analysis for  Algorithm \ref{alg:VP} in the same setting of \citep{antos2008learning,dai2018sbeed} where samples are prefixed and from one single $\beta$-mixing off-policy sample path. We denote $\Hat{L}(\theta_V,\theta_\pi)=\max_{\theta_\rho} L(\theta_V,\theta_\pi,\theta_\rho), s.t., \theta_{\rho^1}=,...,=\theta_{\rho^N} $

\begin{theorem}
	Let the function approximators of $V_i(s)$, $\pi^{i}(s,a^i)$ and $\rho_i(s,a)$ satisfy Assumption \ref{assumption}, snd denote the total training step be $T$. We solve the inner dual problem  with a approximated solution $\tilde{\theta}_{\rho}=(\tilde{\theta}_{\rho^1},...,\tilde{\theta}_{\rho^N})'$, such that $ \| \nabla_{\theta_{V}} L(\theta_V,\theta_\pi,\tilde{\theta}_{\rho}) -\nabla_{\theta_{V}} L(\theta_V,\theta_\pi,\theta^*_{\rho}) \|\leq c_1/\sqrt{T}$, and $ \| \nabla_{\theta_{\pi}} L(\theta_V,\theta_\pi,\tilde{\theta}_{\rho}) -\nabla_{\theta_{\pi}} L(\theta_V,\theta_\pi,\theta^*_{\rho}) \|\leq c_2/\sqrt{T}$. Assume the variance of the stochastic gradient $g(\theta_{V})$, $g(\theta_{\pi})$ and $g(\theta_{\rho})$ (estimated by a single sample) are bounded by $\sigma^2$, the size of the mini-batch is $\sqrt{T}$, the step size $\alpha_{\pi},\alpha_v, \alpha_\rho \propto \frac{1}{L}$.  Then value propagation in Algorithm \ref{alg:VP} converges to the stationary solution of $\Hat{L}(\theta_V,\theta_\pi)$ with rate $\mathcal{O} (1/T).$
\end{theorem}
\textbf{Remarks:} \textbf{(1)} The convergence criteria and its dependence on the network structure are involved. We defer the definition of them to the proof section in the appendix (Equation \eqref{equ:convergence_criteria}).
\textbf{(2)} We require that the approximated dual solution $\tilde{\theta}_\rho$ are not far from $\theta_\rho^*$ such that the estimation of the primal gradient of $\theta_v$ and $\theta_\pi$ are not far from the true one (the distance is less than $\mathcal{O} (1/\sqrt{T})$). Once the inner dual problem is concave, we can get this approximated solution easily using vanilla decentralized stochastic gradient method after at most $T$ steps. If the dual problem is non-convex, we still can show  the dual problem converges to some stationary solution with rate $\mathcal{O}(1/T) $ by our proof.
\textbf{(3)} In the theoretical analysis,  the stochastic gradient estimated from the mini-batch (rather than the estimation from a single sample ) is  common in non-convex analysis, see the work \citep{ghadimi2016accelerated}. In practice, a mini-batch of samples is commonly used in training deep neural network.   
 \section{Related work}
 Among related work on MARL, the setting of \citep{zhang2018fully} is close to ours, where the authors proposed a fully decentralized multi-agent Actor-Critic algorithm to maximize the expected time-average reward $ \lim_{T\rightarrow \infty}\frac{1}{T}\mathbb{E}\sum_{t=1}^{T}\frac{1}{n}\sum_{i=1}^{n} r_i^t$. They provide the \textit{asymptotic} convergence analysis on the \textit{on-policy} and \textit{linear} function approximation setting. In our work, we consider the discounted reward setup, i.e., Equation \eqref{equ:objective}. Our algorithm  includes both on-policy and \textit{off-policy} setting thus can exploit  data more efficiently. Furthermore, we provide a convergence rate  $\mathcal{O}(\frac{1}{T})$  in the \textit{non-linear} function approximation  setting which is much stronger than the result in \citep{zhang2018fully}. \citet{littman1994markov} proposed the framework of  Markov games  which can be applied  to collaborative and competitive setting \citep{lauer2000algorithm,hu2003nash}. These early works considered the tabular case thus can not apply to  real problems with large state space. Recent works \citep{foerster2016learning,foerster2018counterfactual,rashid2018qmix,raileanu2018modeling,jiang2018graph,lowe2017multi} have exploited powerful deep learning and obtained some promising empirical results. However most of them  lacks theoretical guarantees while our work provides convergence analysis.  We emphasize that most of the research on MARL is in the fashion of centralized training and decentralized execution. In the training, they do not have the constraint on the communication, while our work has a network  decentralized  structure.

 \section{Experimental result}
 
The goal of our experiment is two-fold: To better understand the effect of each  component in the proposed algorithm; and to evaluate efficiency of value propagation in the off-policy setting. To this end, we first do an ablation study on a simple random MDP problem, we then evaluate the performance on the cooperative navigation task \citep{lowe2017multi}.  The settings of the experiment are similar to those in \citep{zhang2018fully}. Some implementation details are deferred to Appendix \ref{appendix:implementation_detail} due to space constraints.
 
 \subsection{Ablation Study}\label{section:random_MDP}
 
 \begin{figure*}
 	\centering
 	\includegraphics[width=0.32\textwidth]{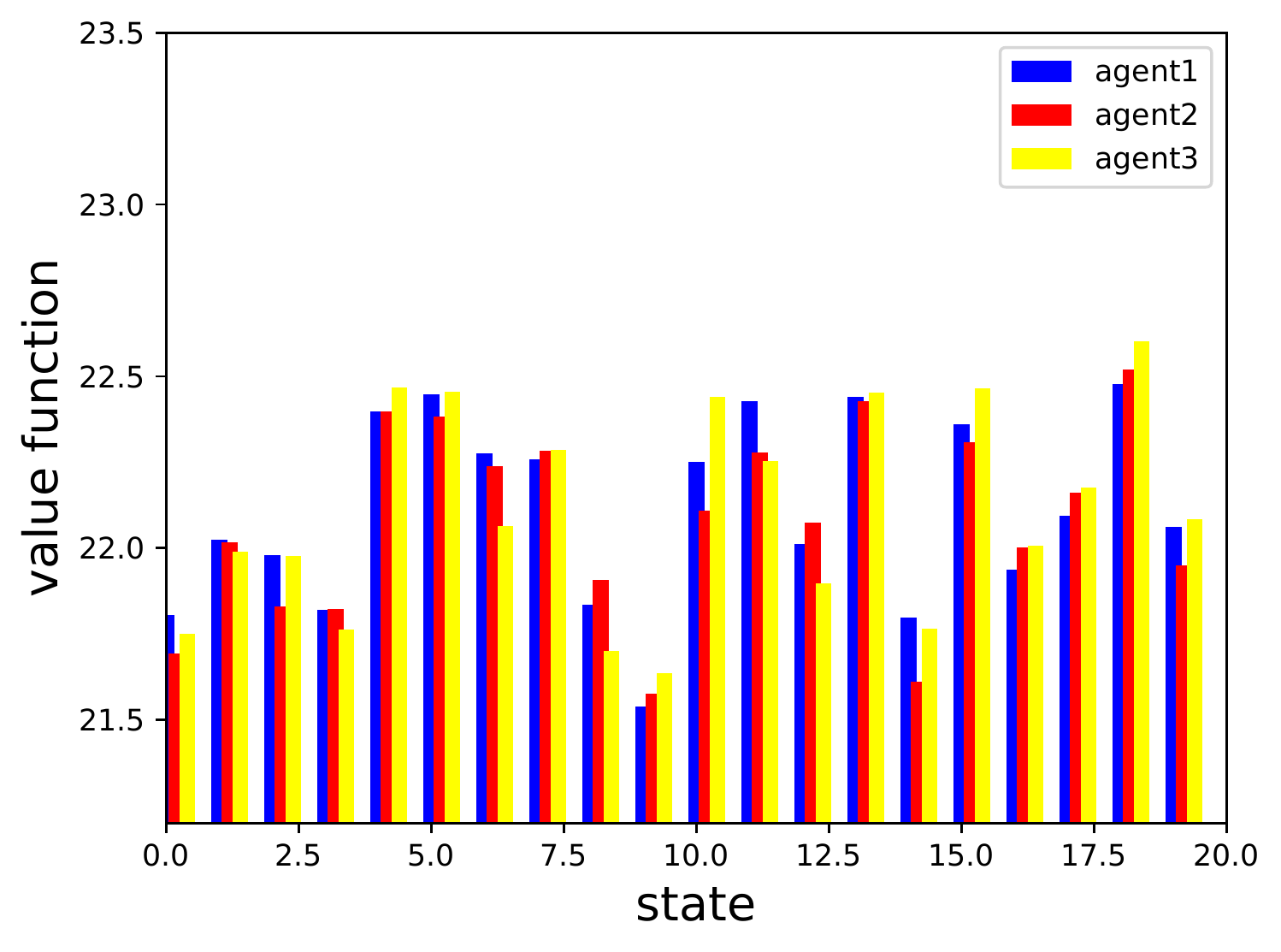}
 	\includegraphics[width=0.32\textwidth]{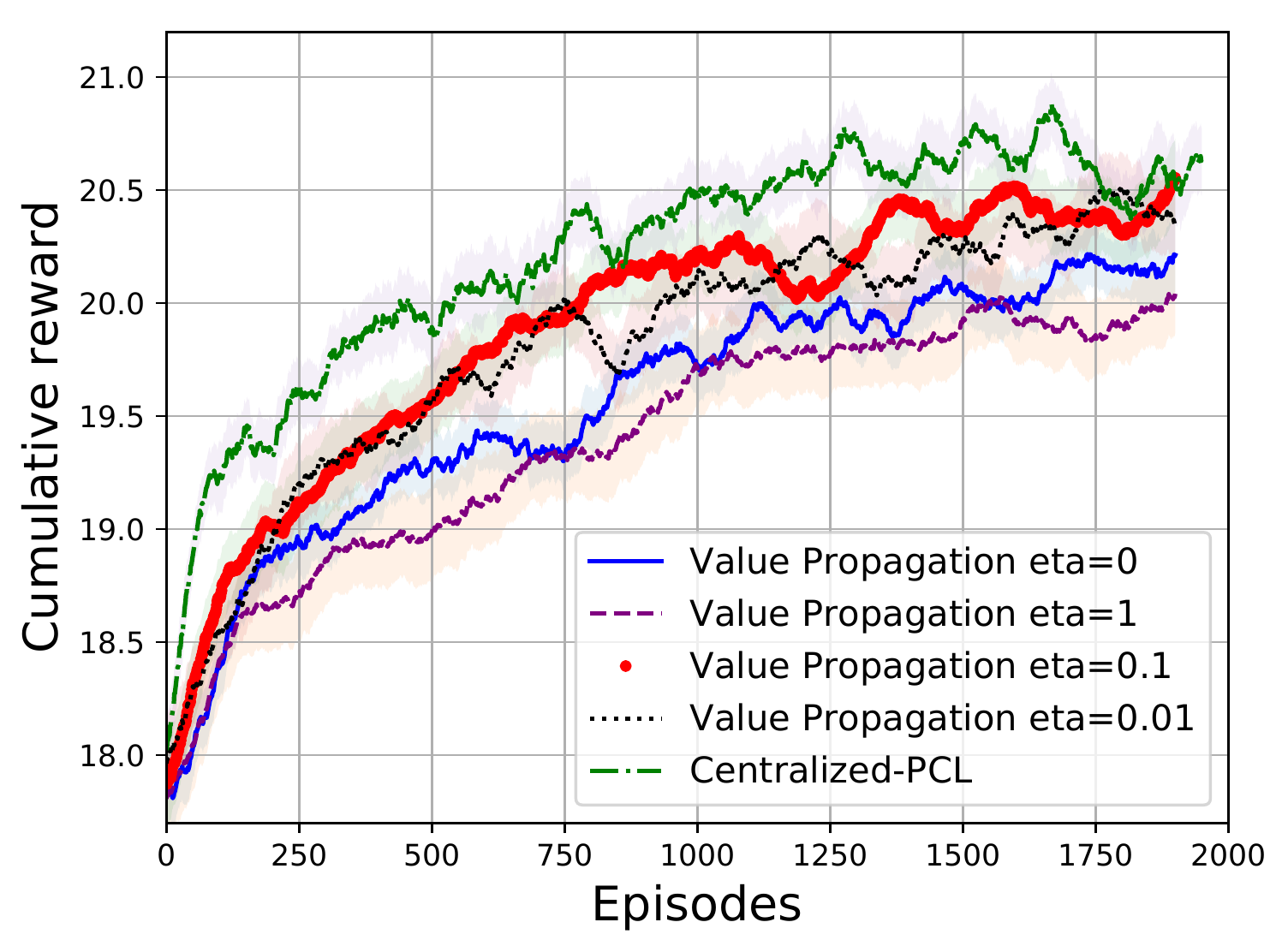}
 	\includegraphics[width=0.32\textwidth]{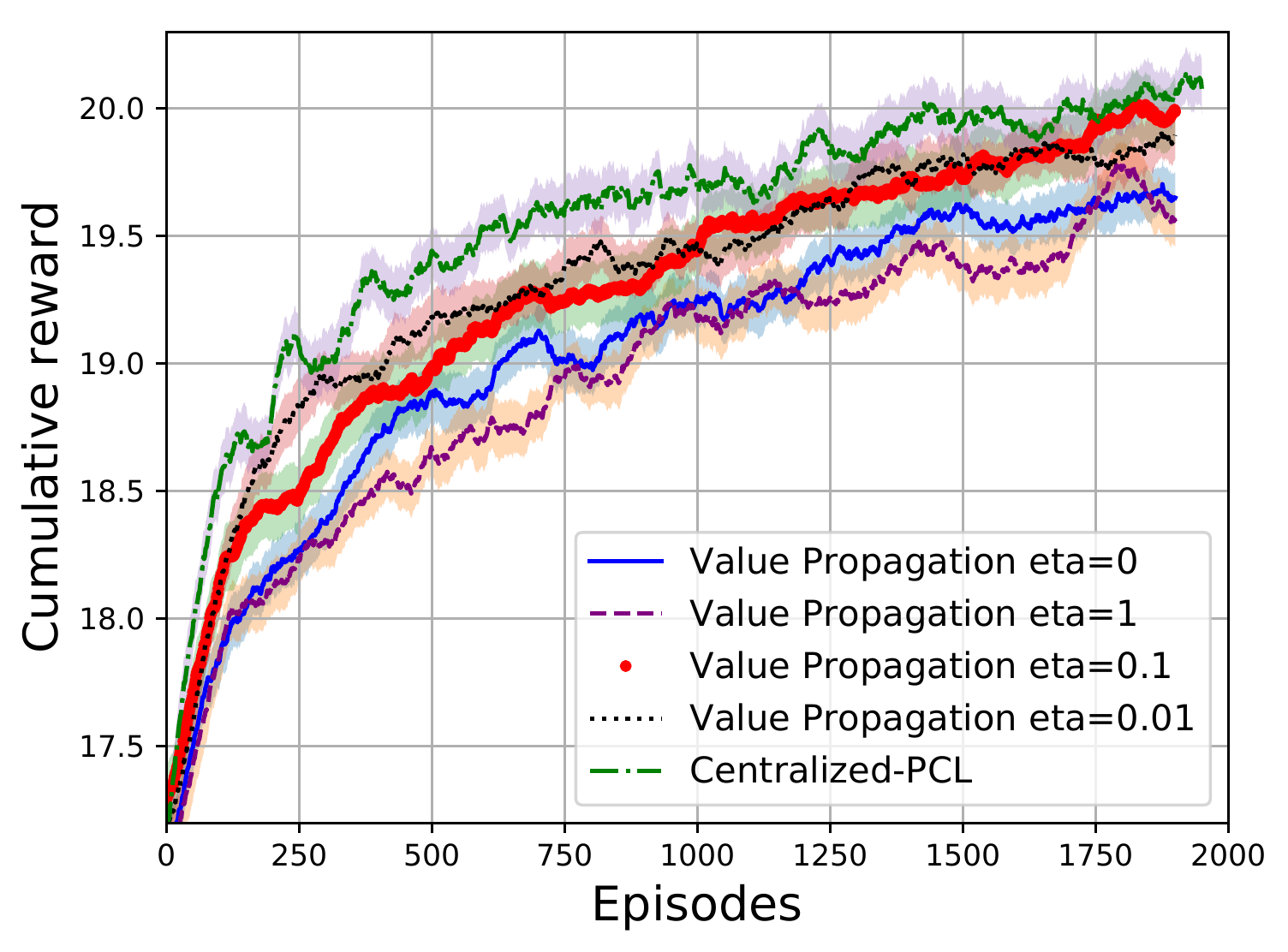}
 	\caption{Results on randomly sampled MDP. Left: Value function of different agents in value propagation. In the figure, value functions of three agents are similar, which means agents get consensus on value functions.  Middle: Cumulative reward of value propagation (with different $\eta$) and centralized PCL with 10 agents.  Right : Results with  20 agents. }\label{Figure:random_mdp}	
 	\vspace{-2mm}
 \end{figure*}
 
  In this experiment, we test effect of several components of our algorithm such as the consensus update, dual formulation in a random MDP problem. Particularly we answer following three questions:
  \textbf{(1)} Whether an agent can get consensus through  message-passing in  value propagation even when each agent just knows its local reward. \textbf{(2) }How much performance does the decentralized approach sacrifice comparing with centralized one? \textbf{(3)} What is the effect of the dual part in our formulation ($0\leq \eta\leq 1 $ and $\eta=0$ corresponds to the pure primal form)?
 
   We compare value propagation with the \textit{centralized} PCL. The centralized PCL means that there is a central node to collect rewards of all agent, thus it can optimize the objective function \eqref{equ:multi_PCL_primal} using the single agent PCL algorithm \citep{nachum2017bridging,dai2018sbeed}.  Ideally, value propagation should converges to the same long term reward  with the one achieved by the centralized PCL. In the experiment, we consider a multi-agent RL problem with $N=10$ and $N=20$ agents, where each agent has two actions.  A discrete MDP is randomly generated with $|\mathcal{S}|=32$ states. The transition probabilities are distributed uniformly with a small additive constant to ensure ergodicity of the MDP, which is $\mathcal{P}(s'|a,s)\propto p_{ss'}^a+10^{-5}, p_{ss'}^a \sim U[0,1]$. For each agent $i$ and each state-action pair $(s,a)$, the reward $R_{i}(s,a)$ is uniformly sampled from $[0,4]$.

 \begin{figure*}[t]
 	\centering
 	\includegraphics[width=0.32\textwidth]{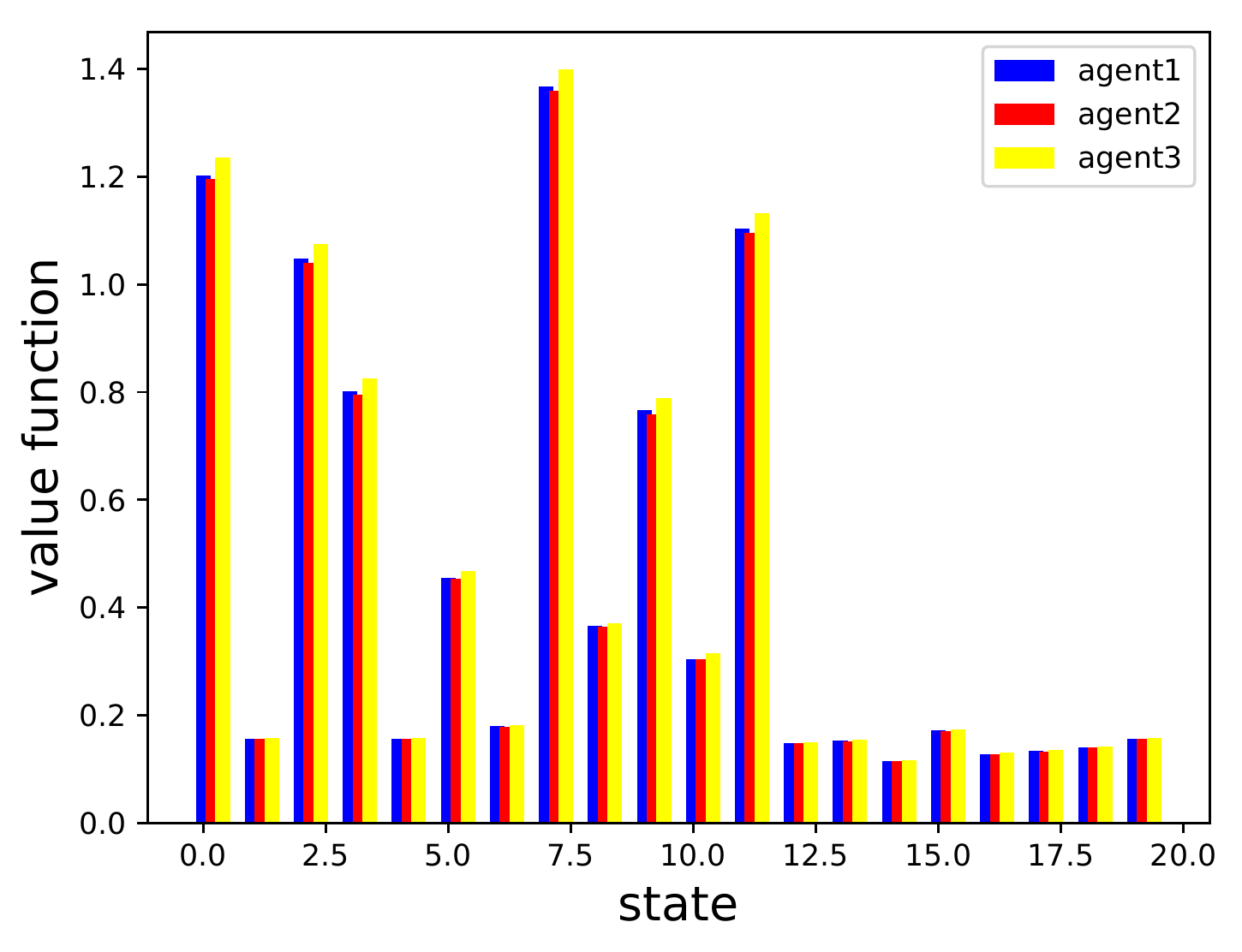}
 	\includegraphics[width=0.32\textwidth]{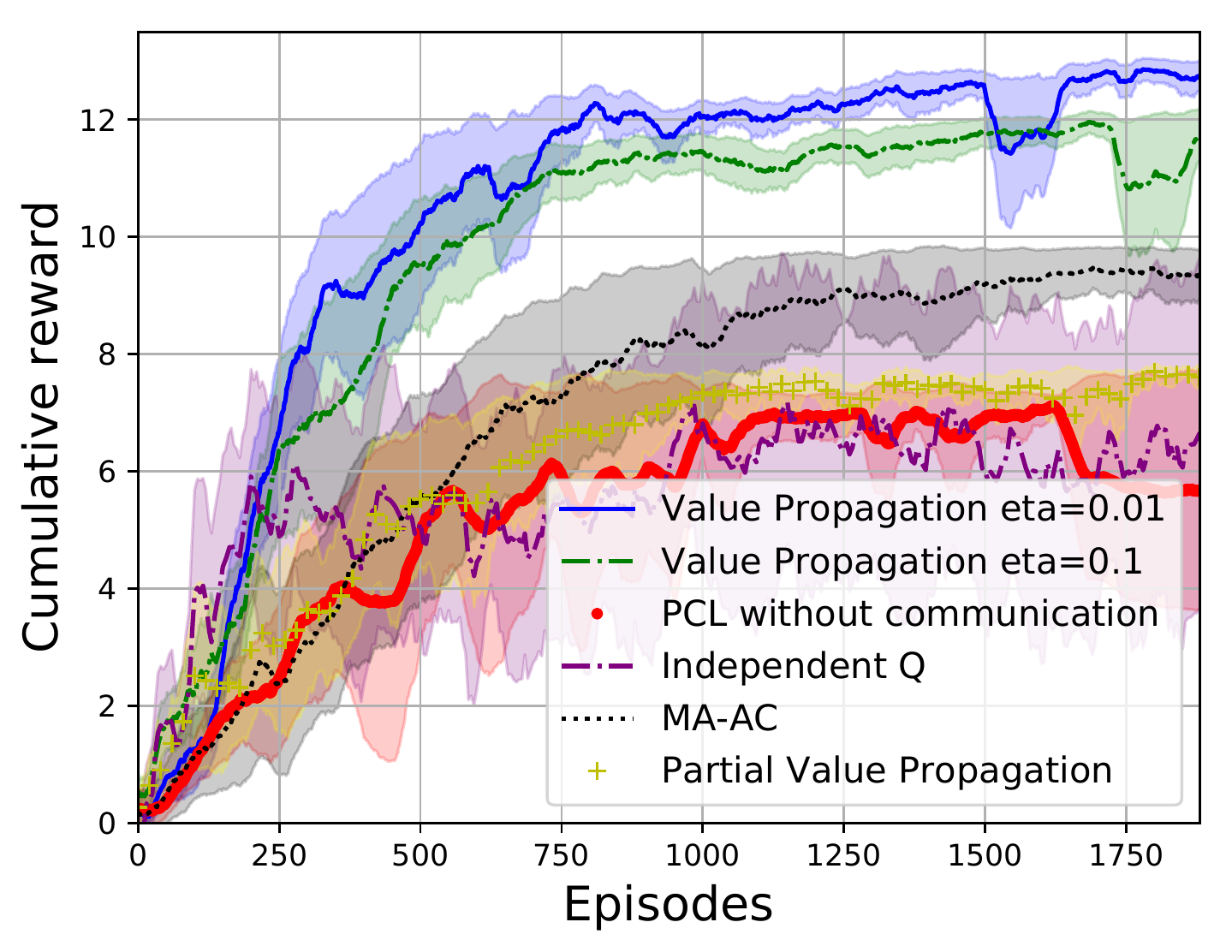}
 	\includegraphics[width=0.32\textwidth]{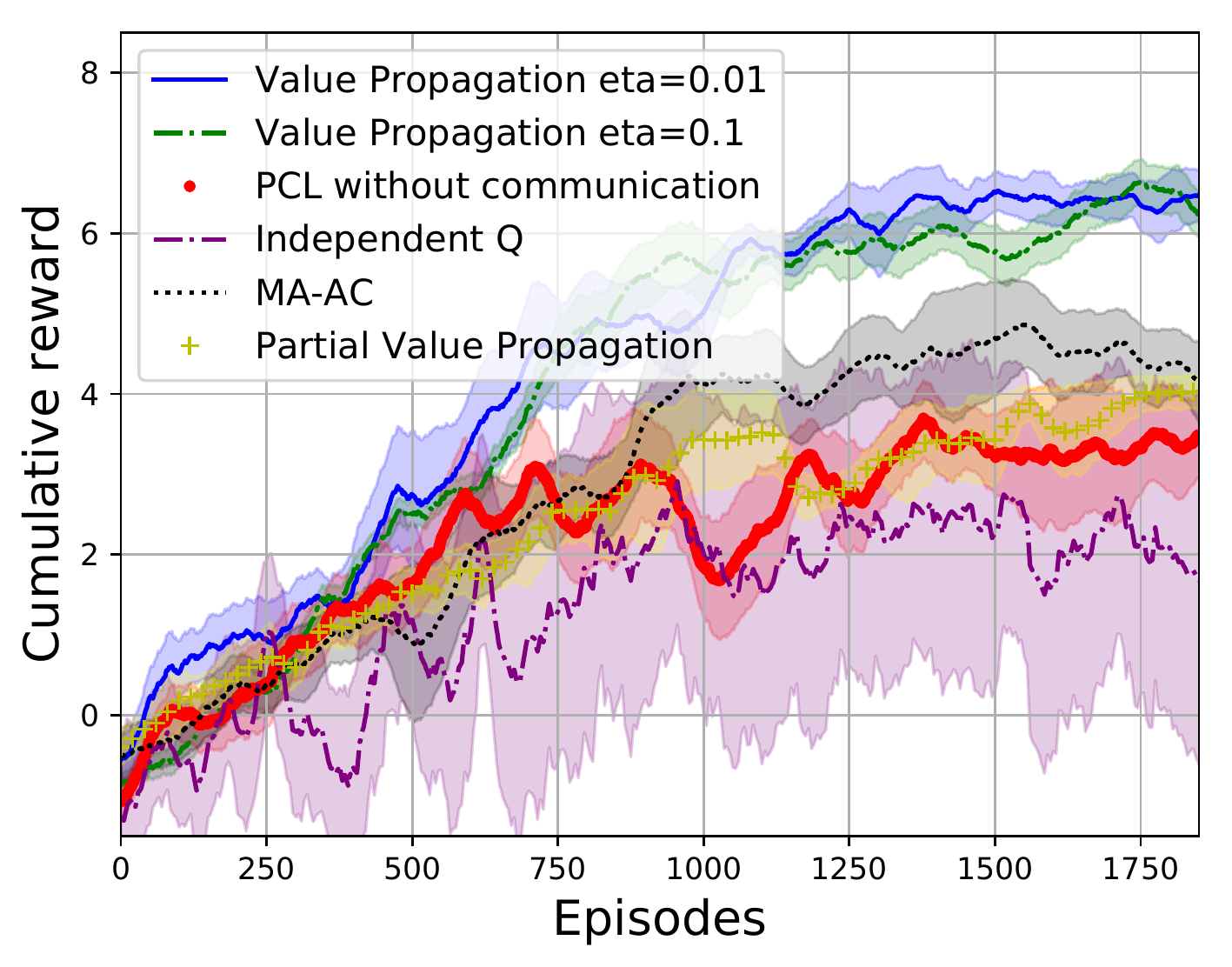}
 	\caption{Results on Cooperative Navigation task.  Left: value functions of three random picked agents (totally 16 agents) in value propagation. They get consensus.  Middle : cumulative reward of value propagation (eta=0.01 and eta=0.1), MA-AC and PCL without communication with agent number N=8. Right: Results with agent number N=16. Our algorithm outperforms MA-AC and PCL without communication.  Comparing with the middle panel, the number of agent increases in the right panel. Therefore, the problem becomes harder (more collisions). We see agents achieve lower cumulative reward (averaged over agents) and need more time to find a good policy.  }\label{Figure:Navigation}
 	\vspace{-2mm}
 \end{figure*}
 
 In the left panel of Figure \ref{Figure:random_mdp}, we verify that the value function $ v_i(s)$ in value propagation reaches the consensus through message-passing in the end of the training. Particularly, we randomly choose three agent $i$, $j$, $k$ and draw their value functions over 20 randomly picked states. It is easy to see that  value functions $v_i (s)$, $v_j(s)$, $v_k(s)$ over these states are almost same. This is accomplished by the consensus update  in value propagation.
 In the middle and right panel of Figure \ref{Figure:random_mdp}, we compare the result of value propagation  with centralized PCL and evaluate the effect of the dual part of value propagation. Particularly, we pick $\eta=0, 0.01, 0.1,1$ in the experiment, where $\eta=0$ corresponds to the pure primal formulation.  When $\eta$ is too large ($\eta=1$), the algorithm would have large variance while $\eta=0$ the algorithm has some bias. Thus value propagation with $\eta=0.1, 0.01$ has better result. We also see that value propagation ($\eta=0.1,0.01$)  and centralized PCL converge to almost the same value, although there is a gap between centralized and decentralized algorithm. The centralized PCL converges faster than value propagation, since it does not need time to diffuse  the reward information over the network.

 \subsection{ Cooperative Navigation task}
 The aim of this section is to demonstrate that the value propagation  outperforms decentralized multi-agent Actor-Critic (MA-AC)\citep{zhang2018fully}, independent Q learning \citep{tan1993multi}, the Multi-agent PCL \textit{without communication}. Here PCL without communication means each agent maintains its own estimation of policy $\pi^i(s,a^i)$ and value function $V^i(s)$ but there is no communication Graph. Notice that this is different from the centralized PCL in Section \ref{section:random_MDP}, where centralized PCL has a central node to collect all reward information and thus do not need further communication. Note that the original MA-AC is designed for the averaged reward setting thus we adapt it into the discounted case to fit our setting.   We test the value propagation in the environment of the Cooperative Navigation task \citep{lowe2017multi}, where agents need to reach a set of $L$ landmarks through physical movement. We modify this environment to fit our setting. A reward is given when the agent reaches its own landmarks. A penalty is received if  agents collide with other agents.  Since the position of landmarks are different, the reward function of each agent is different.  Here we test the case the state is globally observed and partially observed.  In particular, we assume the environment is in a rectangular region with size $2\times 2$. There are $N=8$ or $N=16$ agents. Each agent has a single target landmark, i.e., $L=N$, which is randomly located in the region. Each agent has five actions which corresponds to going up, down, left, right with units 0.1 or staying at the position. The agent has high probability (0.95) to move in the direction following its action and go in other direction randomly otherwise. The maximum length of each epoch is set to be 500 steps.  When the agent is close enough to the landmark, e.g., the distance is less than 0.1, we think it reaches the target and gets reward $+5$.   When two agents are close to each other (with distance less than 0.1), we treat this case as a collision and a penalty $-1$ is received for each of the agents. The state includes the position of the agents. The communication graph is generated as that in Section \ref{section:random_MDP} with connectivity ratio $4/N$. In the partially observed case, the actor of each agent can only observe  its own and  neighbors' states. We report the results in Figure \ref{Figure:Navigation}.

In the left panel of Figure \ref{Figure:Navigation}, we see the value function $v_i(s)$ reaches consensus in value propagation. In the middle and right panel of Figure \ref{Figure:Navigation}, we compare value propagation  with  PCL without communication, independent Q learning and MA-AC. In  PCL without communication, each agent maintains its own policy, value function and dual function, which is trained by the algorithm SBEED \citep{dai2018sbeed} with $\eta=0.01$. Since there is no communication between agents, intuitively agents may have more collisions in the learning process than those in value propagation. Similar augment holds for the independent Q learning.  Indeed, In the middle and right panel, we  see value propagation learns the policy much faster than PCL without communication. We also observe that value propagation outperforms MA-AC. One possible reason is that value propagation is an off-policy method thus we can apply experience replay which exploits data more efficiently than the on-policy method MA-AC.  We also test the performance of value propagation (result labeled as partial value propagation in Figure \ref{Figure:Navigation}) when the state information of actor is partially observed. Since the agent has limited information, its performance is worse than the fully observed case. But it is better than the PCL without communication  (fully observed state).

\bibliography{MARL}
\bibliographystyle{plainnat}

\clearpage

\appendix

\section{Details on the Value Propagation}

\subsection{Topology of the Graph}\label{appendix:topology}
Here, we explain the matrix in the Algorithm \ref{alg:VP} which are closely related to the topology of the Graph, which is left from the main paper due to the limit of the space.

\begin{itemize}
	\item  $D=diag[d_1,...,d_N]$ is the degree matrix, with $d_i$ denoting the degree of node $i$. 
	\item  $A$ is the node-edge incidence matrix: if $e\in \mathcal{E}$ and it connects vertex i and j with $i>j$, then $A_{ev}=1$ if $v=i$, $A_{ev}=-1$ if $v=j$ and $A_{ev}=0$ otherwise.
	\item The signless incidence matrix $B:=|A|$, where the absolute value is taken for each component of $A$. 
	\item  The signless graph Laplacian $L^{+}=B^TB$. By definition $L^{+}(i,j)=0$  if $(i,j)\notin \mathcal{E}$.  Notice the non-zeros element in $A$, $L^+$, the update just depends on each agent itself and its neighbor.  
	
\end{itemize}

\subsection{Practical Acceleration}\label{appendix:acceleration}
The algorithm \ref{alg:VP} trains the agent with vanilla gradient decent method with a extra consensus update. In practice, the adaptive momentum gradient methods including Adagrad \cite{duchi2011adaptive}, Rmsprop \cite{tielemandivide} and Adam \cite{kingma2014adam} have much better performance in training the deep neural network. We adapt Adam in our setting, and propose algorithm \ref{alg:VP2} which has  better performance than algorithm \ref{alg:VP} in practice.

\begin{algorithm}[htb!]
	\caption{Accelerated value propagation }
	\label{alg:VP2}
	\begin{algorithmic}
		\STATE { Input: Environment ENV, learning rate $\beta_1,\beta_2 \in [0,1)$, $\alpha_t$, discount factor $\gamma$, a mixing matrix $W$, number of step $T_{dual}$ to train dual parameter $\theta_\rho^i$,   replay buffer capacity $B$}.	
		\STATE {Initialization of $\theta_{v^i}, \theta_{\pi^i}, \theta_{\rho^i}$, moment vectors $m^0_{v^i}=m^0_{\rho^i}=0$, $w^0_{v^i}=w^0_{\rho^i} =0$ }. 
		\FOR{$t=1,...,T$}
		\STATE{ sample trajectory $s_{0:\tau}\sim \pi(s,a)=\prod_{i=1}^{N}\pi^i(s,a^i) $ and add it into the replay buffer.}\\		
		//\textbf{ Update the dual parameter} $\theta_{\rho^i}$\\
		Do following update $T_{dual}$ times:
		\STATE{Random sample a mini-batch of transition $(s_t, \{a^i_t\}_{i=1}^{N},s_{t+1}, \{r^i_t\}_{i=1}^N)$ from the replay buffer.}
		
		\FOR {agent $i=1$  to $n$}
		\STATE{Calculate the stochastic gradient $g(\theta^t_{\rho^i})$ of $-\eta(\delta_i(s_t,a_t,s_{t+1})-\rho_i(s_t,a_t))^2$ w.r.t. $\theta^t_{\rho^i}$. \\
			// update  momentum parameters:
			$m_{\rho^i}^{t}=\beta_1 m^{t-1}_{\rho^i}+(1-\beta_1) (-g(\theta^t_{\rho^i})) $\\
			$w^{t}_{\rho^i}=\beta_2 w^{t-1}_{\rho^i}+(1-\beta_2) g(\theta^t_{\rho^i}) \odot g(\theta^t_{\rho^i}) $ }
		\ENDFOR \\
		// Do consensus update for each agent $i$ \\
		$\theta^{t+\frac{1}{2}}_{\rho^i}=\sum_{j=1}^{N}[W]_{ij}\theta^t_{\rho^j},$ $\theta^t_{\rho^i}=\theta^{t+\frac{1}{2}}_{\rho^i}-\alpha_t \frac{m^t_{\rho^i}}{\sqrt{w^t_{\rho^i}}}$

		// \textbf{End the update of dual problem}\\
		// \textbf{Update primal parameters} $\theta_{v^i}, \theta_{\pi^i}$.
		\STATE{Random sample a mini-batch of transition $(s_t, \{a^i_t\}_{i=1}^{N},s_{t+1}, \{r^i_t\}_{i=1}^N)$ from the replay buffer.}
		\FOR{ agent $i=1$ to $n$ }
		\STATE{Calculate the stochastic gradient $g(\theta^t_{v^i})$,$g(\theta^t_{\pi^i})$ of $(\delta_i(s_t,a_t,s_{t+1})-V_i(s_t))^2-\eta (\delta_i(s_t,a_t,s_{t+1})-\rho_i(s_t,a_t))^2 $,  w.r.t. $\theta^t_{v^i}$, $\theta^t_{\pi^i}$\\
			//update the momentum parameter:\\
			$m_{v^i}^{t}=\beta_1 m^{t-1}_{v^i}+(1-\beta_1) g(\theta^t_{v^i}) $\\
			$w^{t}_{v^i}=\beta_2 w^{t-1}_{v^i}+(1-\beta_2) g(\theta^t_{v^i}) \odot g(\theta^t_{v^i}) $	
			
		}
		
		// Using Adam to update $ \theta_{\pi^i}$ for each agent $i$.\\
		// Do consensus update on $\theta_{v^i}$ for each agent $i$:\\
		$\theta^{t+\frac{1}{2}}_{v^i}=\sum_{j=1}^{N}[W]_{ij}\theta^t_{v^j},$ $\theta^t_{v^i}=\theta^{t+\frac{1}{2}}_{v^i}-\alpha_t \frac{m^t_{v^i}}{\sqrt{w^t_{v^i}}}$		
		\ENDFOR
		\ENDFOR
	\end{algorithmic}
\end{algorithm}

\mbox{\bf Mixing Matrix:} In Algorithm \ref{alg:VP2}, there is a mixing matrix $W\subset R^{N\times N}$ in the consensus update. As its name suggests, it mixes information of the agent and its neighbors. This nonnegative matrix $W$ need to satisfy the following condition.

\begin{itemize}
	\item $W$ needs to be doubly stochastic, i.e., $W^T\textbf{1}=\textbf{1}$ and $W \textbf{1}=\textbf{1}$.
	\item $W$ respects the communication graph $\mathcal{G}$, i.e., $W(i,j)=0$ if $(i,j)\notin \mathcal{E}$.
	\item The spectral norm of $W^T (I-\textbf{1}\textbf{1}^T/N
	)W $ is strictly smaller than one.
\end{itemize}
Here is one particular choice of the mixing matrix $W$ used in our work which satisfies above requirement called Metropolis weights \cite{xiao2005scheme}.
\begin{equation}\label{equ:mixing_matrix}
\begin{split}
W(i,j)={1+\max[d(i), d(j)] }^{-1}, \forall (i,j)\in \mathcal{E},\\
W(i,i)=1-\sum_{j\in NE(i)} W(i,j), \forall i\in\mathcal{N},
\end{split}
\end{equation}
where $NE(i)=\{j\in \mathcal{N}: (i,j)\in \mathcal{E}\}$ is the set of neighbors of the agent $i$ and $d(i)=|\mathcal{N}(i)|$ is the degree of agent $i$. Such mixing matrix is widely used in decentralized and distributed optimization \cite{boyd2006randomized,cattivelli2008diffusion}. The update rule of the momentum term in Algorithm \ref{alg:VP2} is adapted from Adam. The consensus (communication)  steps  are 	$\theta^{t+\frac{1}{2}}_{\rho^i}=\sum_{j=1}^{N}[W]_{ij}\theta^t_{\rho^j}$ and $\theta^{t+\frac{1}{2}}_{v^i}=\sum_{j=1}^{N}[W]_{ij}\theta^t_{v^j}$.

\subsection{Multi-step Extension on value propagation}\label{appendix:multi-step}
The temporal consistency can be extended to the multi-step case \cite{nachum2017bridging}, where the following equation holds
\begin{equation*}
\begin{split}
V(s_0)=\sum_{t=0}^{k-1}\gamma^t \mathbb{E}_{s_t|s_0,a_{0:t-1}}[R(s_t,a_t)-\lambda \log\pi(s_t,a_t^i)]
+\gamma^k \mathbb{E}_{s_k|s_0,a_{0:k-1}} V(s_k).
\end{split}
\end{equation*}

Thus in the objective function \eqref{equ:final_objective_VP}, we can replace $
\delta_i$ by $\delta_i(s_{0:k},a_{0:k-1})=\sum_{t=0}^{k-1} \gamma^t \big( R_i(s_t,a_t)-\lambda N\log\pi^i(s_t,a_t^i) \big)+\gamma^k V_i(s_k) $  and  change the estimation of stochastic gradient correspondingly in Algorithm \ref{alg:VP} and Algorithm \ref{alg:VP2} to get the multi-step version of vaue propagation . In practice, the performance of setting $k>1$ is better than $k=1$ which is also observed in single agent case \cite{nachum2017bridging,dai2018sbeed}. We can tune $k$ for each application to get the best performance.

\subsection{Implementation details of the experiments}\label{appendix:implementation_detail}

\textbf{Ablation Study}

The value function $v_i(s)$ and dual variable $\rho_i(s,a)$ are approximated by two hidden-layer neural network with Relu as the activation function where each hidden-layer has $20$ hidden units. The policy of each agent is approximated by a one hidden-layer neural network with Relu as the activation function where the number of the hidden units is $32$. The output is the softmax function to approximate $\pi^i(s,a^i)$.  The mixing matrix in Algorithm \ref{alg:VP2} is selected as  the Metropolis Weights in \eqref{equ:mixing_matrix}. The graph  $\mathcal{G}$ is generated by randomly placing communication links among agents such that the connectivity ratio is $4/N$. We set $\gamma=0.9$, $\lambda=0.01$, learning rate $\alpha$=5e-4. The choice of $\beta_1$, $\beta_2$ are the default value in Adam.

\textbf{Cooperative Navigation task} 

The value function $v_i(s)$ is approximated by a two-hidden-layer neural network with Relu as the activation function where  inputs are the state information. Each hidden-layer has $40$ hidden units. The dual function $\rho(s,a)$ is also approximated by a two-hidden-layer neural network, where the only difference is that  inputs are state-action pairs (s,a).  The policy is approximated by a one-hidden-layer neural network with Relu as the activation function. The number of the hidden units is $32$. The output is the softmax function to approximate $\pi^i(s,a^i)$. In all experiments, we use the multi-step version of value propagation and choose $k=4$. We choose $\gamma=0.95$, $\lambda=0.01$. The learning rate of Adam is chosen as 5e-4 and $\beta_1, \beta_2$ are default value in Adam optimizer. The setting of PCL without communication is exactly same with value propagation except the absence of communication network.

\subsection{Consensus update in Algorithm \ref{alg:VP} }\label{appendix:update_rule}

We now give details to derive the Consensus Update in Algorithm \ref{alg:VP} with $\eta=1$ to ease the exposition. When $\eta\in [0,1)$, we just need to change variable and some notations, the result are almost same. Here we use the primal update as an example, the derivation of the dual update is the same.

In the main paper section \ref{section:VP}, we have shown that when $\eta=1$, in the primal update, we basically solve following problem.

\begin{equation}\label{equ:appendix_dual}
\begin{split}
&\min_{  \{\theta_{v_i},\theta_{\pi^i}\}_{i=1}^{N} } 2\mathbb{E}_{s,a,s} [\nu^*(s,a) \big( \frac{1}{N}\sum_{i=1}^{N} ( R^i(s,a)+ \gamma V_i(s')
-V_i(s)
-\lambda N\log\pi^i(s,a) \big)]-\mathbb{E}_{s,a,s}[\nu^{*2}(s,a)],\\ &s.t., \theta_{v_1}=...=\theta_{v_n}.
\end{split}
\end{equation} 

here for simplicity we assume in the dual optimization, we have already find the optimal solution $\nu^*(s,a)$. It can be any approximated solution of $\tilde{\nu}(s,a)$ which does not affect the derivation of the update rule in primal optimization. In the later proof, we will show how this approximated solution affects the convergence rate. 

When we optimize w.r.t. $\theta_{v^i}$, we basically we solve a non-convex problem with the following form

\begin{equation}\label{equ:abstract_obj}
\min_{x} f(x)=\sum_{i=1}^{N}  f_i(x_i),~~s.t.~~ x_1=...=x_N
\end{equation}

Recall the definition of the node-edge incidence matrix $A$: if $e\in \mathcal{E}$ and it connects vertex $i$ and $j$ with $i>j$, then $A_{ev}=1$ if $v=i$, $A_{ev}=-1$ if $v=j$ and $A_ev=0$ otherwise. Thus by define $x=[x_1,...,x_N]'$ we have a equivalent form of \eqref{equ:abstract_obj}

\begin{equation}\label{equ:abstract_obj_constraint}
\min_{x} f(x)=\sum_{i=1}^{N} f_i(x_i), ~~s.t., ~~Ax=0
\end{equation}

Notice the update  of $\theta_{\pi^i}$ is a special case of above formulation, since we do not have the constraint $x_1=,...,=x_N$. Thus in the following, it suffice to analyze above formulation \eqref{equ:abstract_obj_constraint}. We adapt the Prox-PDA in \cite{hong2017prox} to solve above problem. To keep the notation consistent with \cite{hong2017prox}, we consider a more general problem

$$\min_xf(x)=\sum_{i=1}^{N}f_i(x_i), s.t., Ax=b.$$

In the following  we denote $\nabla f(x^t):=[(\nabla_{x_1} f(x_1))^T,...,(\nabla_{x_N}f(x_N))^T]^T$ where the superscript $'$ means transpose. We denote $g_i(x_i)$ as an estimator of $\nabla_{x_i} f(x_i)$ and $g(x)=[g_1(x_1),...,g_N(x_N)]$.

The update rule of Prox-PDA is 

\begin{equation}\label{alg:general1}
x^{t+1}=\arg \min_{x} \langle g(x^{t}), x-x^{t} \rangle+\langle \mu^{t}, Ax-b \rangle+\frac{\beta}{2}\|Ax-b\|^2+\frac{\beta}{2}\|x-x^t\|^2_{B^T B}
\end{equation}

\begin{equation}\label{alg:gnenral2}
\mu^{t+1}=\mu^t+\beta(Ax^{t+1}-b)
\end{equation}
where $g(x^t)$ is an estimator of $\nabla f(x^{t})$. The signed graph Laplacian matrix $L_{-}$ is $A^TA$. Now we choose $B:=|A|$ as the signless incidence matrix. Using this choice of $B$, we have $B^TB=L^{+}\in \mathbb{R}^{N\times N}$ which is the signless graph Laplacian whose $(i,i)$th diagonal entry is the degree of node $i$ and its $(i,j)$th entry is 1 if $e=(i,j)\in\mathcal{E}$, and 0 otherwise.

Thus
\begin{equation}
\begin{split}
x^{t+1}=&\arg\min_{x} \langle g(x^t),x-x^t \rangle+\langle \mu^{t}, Ax-b \rangle+\frac{\beta}{2}x^TL_{-}x+\frac{\beta}{2}(x-x^t)L^{+}(x-x^t)\\
=&\arg\min_{x} \langle g(x^t),x\rangle+\langle \mu^{t}, Ax-b \rangle+\frac{\beta}{2}x^T(L_{-}+L^{+})x-\beta x^TL^{+}x^t\\
=&\arg\min_{x} \langle g(x^t),x\rangle+\langle \mu^{t}, Ax-b \rangle+\beta x^TDx-\beta x^TL^{+}x^t,
\end{split}
\end{equation}
where $D=diag[d_1,...,d_N]$ is the degree matrix, with $d_i$ denoting the degree of node $i$.

After simple algebra, we obtain

$$ x^{t+1}=\frac{1}{2}D^{-1}L^{+}x^t-\frac{1}{2\beta} D^{-1}A^T \mu^t-\frac{1}{2\beta}D^{-1} g(x^t), $$ which is the primal update rule of the consensus step in the algorithm \ref{alg:VP} (notice here the stepsize is $1/\beta$)

\section{Convergence Proof of Value Propagation}

\subsection{Convergence on the primal update}
In this section, we first give the convergence analysis of the value propagation (algorithm \ref{alg:VP}) on the primal update. To include the effected of the inexact solution of dual optimization problem, we denote $g(x^t)=\nabla f(x^t)+\epsilon_t$, where $\epsilon_t=\varepsilon_t+ \tilde{\varepsilon}_t$ is some error terms.  

\begin{itemize}
	\item $\varepsilon_t$ is a zero mean random variable coming from the randomness of the stochastic gradient $g(x^t)$.
	\item  $\tilde{\varepsilon}_t$ comes from the approximated solution of $\tilde{\nu}$ in \eqref{equ:appendix_dual} or $\tilde{\rho}$ in \eqref{equ:final_objective_VP} such that  $ \| \nabla_{\theta_{v}} L(\theta_V,\theta_\pi,\tilde{\theta}_{\rho}) -\nabla_{\theta_{v}} L(\theta_V,\theta_\pi,\theta^*_{\rho}) \|\leq \tilde{\varepsilon}_t$ and $ \| \nabla_{\theta_{\pi}} L(\theta_V,\theta_\pi,\tilde{\theta}_{\rho}) -\nabla_{\theta_{\pi}} L(\theta_V,\theta_\pi,\theta^*_{\rho}) \|\leq \tilde{\varepsilon}_t.$
\end{itemize}  

Before we begin the proof, we made some mild assumption on the function $f(x)$.

\begin{assumption}\label{assumption:f}
	1. The function f(x) is differentiable and has Lipschitz continuous gradient, i.e., 
	$$ \|\nabla f(x)-\nabla f(y)\|\leq \|x-y\|, \forall x,y \in \mathbb{R}^K.$$
	
	2. Further assume that $A^TA+B^TB\succcurlyeq I$. 
	This assumption is always satisfied by our choice on A and B. We have $A^TA+B^TB\succcurlyeq D\succcurlyeq \min_i\{d_{i}\} I$
	
	3. There exists a constant $\delta>0$ such that 
	$\exists \underline{f}>-\infty, s.t., f(x)+\frac{\delta}{2}\|Ax-b\|^2\geq \underline{f}, \forall x.$ This assumption is satisfied if we require the parameter space is bounded.

\end{assumption}

\begin{lemma}\label{lemma:lemma_on_mu}
	Suppose the assumption \ref{assumption:f} is satisfied, we have following inequality holds
	\begin{equation}
	\frac{\|\mu^{t+1}-\mu^{t}\|^2}{\beta}\leq \frac{3L^2}{\beta \sigma_{\min}}\|x^{t}-x^{t-1}\|^2+\frac{3}{\beta}\|\epsilon_{t-1}-\epsilon_t\|^2+\frac{3\beta}{\sigma_{\min}} \|B^TB \big( (x^{t+1}-x^t)-(x^t-x^{t-1})  \big)\|^2	
	\end{equation}
\end{lemma}

\begin{proof}
	Using the optimality condition of \eqref{alg:general1}, we obtain
	
	\begin{equation}
	\nabla f(x^{t})+\epsilon^t+A^T \mu^t+\beta A^T(Ax^{t+1}-b)+\beta B^TB(x^{t+1}-x^t)=0
	\end{equation}

	applying equation \eqref{alg:gnenral2} we have
	
	\begin{equation}
	A^T\mu^{t+1}=-\nabla f(x^{t})-\beta B^TB (x^{t+1}-x^{t}).
	\end{equation}
	
	Note that from the fact that $\mu^0=0$, we have the variable lies in the column space of $A$.
	
	$$ \mu^r=\beta\sum_{t=1}^{r} (Ax^t-b).$$
	
	Let $\sigma_{\min}$ denote the smallest non-zero eigenvalue of $A^TA$, we have
	
	\begin{equation}
	\begin{split}
	&\sigma_{\min}^{1/2} \|\mu^{t+1}-\mu^t\|\\
	\leq & \|A(\mu^{t+1}-\mu^t)\|\\
	\leq &\|-\nabla f(x^t)-\epsilon_t-\beta B^TB(x^{t+1}-x^t)-(-\nabla f(x^{t-1})-\epsilon_{t-1}-\beta B^TB(x^{t}-x^{t-1})) \|\\
	=&\|\nabla f(x^{t-1})-\nabla f(x^t)+(\epsilon_{t-1}-\epsilon_t) -\beta B^TB( (x^{t+1}-x^{t}-(x^t-x^{t-1})) )\|.\\
	\end{split}
	\end{equation}
	
	Thus we have

	\begin{equation}
	\begin{split}
	&\frac{\|\mu^{t+1}-\mu^{t}\|^2}{\beta}\\ &\leq \frac{1}{\beta \sigma_{\min}^{1/2}} \|\nabla f(x^{t-1})-\nabla f(x^t)+(\epsilon_{t-1}-\epsilon_t) -\beta B^TB( (x^{t+1}-x^{t}-(x^t-x^{t-1})) )\|^2\\
	&\leq \frac{3L^2}{\beta \sigma_{\min}}\|x^{t}-x^{t-1}\|^2+\frac{3}{\beta}\|\epsilon_{t-1}-\epsilon_t\|^2+\frac{3\beta}{\sigma_{\min}} \|B^TB \big( (x^{t+1}-x^t)-(x^t-x^{t-1})  \big)\|^2,		
	\end{split}
	\end{equation}
	
	where the second inequality holds from the fact that $(a+b+c)^2\leq 3a^2+3b^2+3c^2$.
	
\end{proof}	
\begin{lemma}
	Define $L_\beta(x^t,\mu^t)=f(x^t)+\langle \mu^{t}, Ax-b \rangle+\frac{\beta}{2}\|Ax-b\|^2+\frac{\beta}{2}\|x-x^t\|_{B^T B}  $. 
	Suppose assumptions are satisfied, then the following is true for the algorithm

	\begin{equation}
	\begin{split}
	&L_\beta (x^{t+1}, \mu^{t+1})-L_\beta(x^{t}, \mu^t) \\ \leq & -\frac{\beta}{2}\|x^{t+1}-x^t\|^2+ \frac{3L^2}{\beta \sigma_{\min}}\|x^{t}-x^{t-1}\|^2+\frac{3}{\beta}\|\epsilon_{t-1}-\epsilon_t\|^2+\frac{3\beta}{\sigma_{\min}} \|B^TB \big( (x^{t+1}-x^t)-(x^t-x^{t-1})  \big)\|^2\\
	&+\langle \epsilon_t, x^{t+1}-x^t \rangle				
	\end{split}
	\end{equation}
	
\end{lemma}

\begin{proof}

	By the Assumptions $A^TA+B^TB \geq I$, the objective function in \eqref{alg:general1} is strongly convex with parameter $\beta$.

	Using the optimality condition of $x^{t+1}$ and strong convexity, we have for any $x$,
	
	\begin{equation}
	\begin{split}
	&L_\beta(x,\mu^t)+\frac{\beta}{2}\|x-x^t\|^2_{B^TB}- (L_\beta (x^{t+1},\mu^t )+ \frac{\beta}{2}\|x^{t+1}-x^t\|^2_{B^TB})\\
	&\geq \langle \nabla L_\beta (x^{t+1}, \mu^t)+\beta B^TB (x^{t+1}-x^t), x-x^{t+1} \rangle +\frac{\beta}{2}\|x^{t+1}-x\|^2
	\end{split}
	\end{equation}

	Now we start to provide a upper bound of $L_\beta (x^{t+1}, \mu^{t+1})-L_\beta(x^t,\mu^t)$ .
	
	\begin{equation}
	\begin{split}
	&L_\beta (x^{t+1}, \mu^{t+1})-L_\beta(x^{t}, \mu^t)\\
	=&L_\beta(x^{t+1},\mu^{t+1})-L_\beta(x^{t+1},\mu^t)+L_{\beta}(x^{t+1},\mu^t)-L_\beta (x^t,\mu^t)\\
	\leq & L_\beta(x^{t+1},\mu^{t+1})-L_\beta(x^{t+1},\mu^t)+L_\beta(x^{t+1},\mu^t)+\frac{\beta}{2}\|x^{t+1}-x^t\|^2_{B^TB}-L_\beta(x^t,\mu^t)\\
	\overset{a}{\leq} & \frac{\|\mu^{t+1}-\mu^t\|}{\beta}+\langle \nabla L_\beta (x^{t+1},\mu^t)+\beta B^TB(x^{t+1}-x^t), x^{t+1}-x^t\rangle-\frac{\beta}{2}\|x^{t+1}-x^t\|^2\\
	\overset{b}{\leq} & -\frac{\beta}{2} \|x^{t+1}-x^t\|^2+\frac{\|\mu^{t+1}-\mu^t\|^2}{\beta}+ \langle \epsilon_t, x^{t+1}-x^t \rangle\\
	\overset{c}{\leq} & -\frac{\beta}{2}\|x^{t+1}-x^t\|^2+ \frac{3L^2}{\beta \sigma_{\min}}\|x^{t}-x^{t-1}\|^2+\frac{3}{\beta}\|\epsilon_{t-1}-\epsilon_t\|^2\\	
	&+\frac{3\beta}{\sigma_{\min}} \|B^TB \big( (x^{t+1}-x^t)-(x^t-x^{t-1})  \big)\|^2+\langle \epsilon_t, x^{t+1}-x^t \rangle,
	\end{split}
	\end{equation}
	where the inequality (a) holds from the update rule in \eqref{alg:gnenral2} and a simple algebra from the expression of $L_\beta(x,\mu)$. 
	Inequality (b) comes from  the optimality condition of \eqref{alg:general1}. Particularly, we have 
	
	$$ g(x^t)+A^T \mu^t+\beta A^T(Ax-b)+\beta B^TB(x-x^t)=0$$ 
	
	replace $g(x^t)$ by $\nabla f(x^t)+\epsilon_t$, we have the result. The inequality (c) holds using the Lemma \ref{lemma:lemma_on_mu}.
	
\end{proof}
\begin{lemma}
	Suppose Assumption \ref{assumption:f} is satisfied, then the following condition holds.	
	\begin{equation}
	\begin{split}
	&\frac{\beta}{2} ( \|Ax^{t+1}-b\|^2+\|x^{t+1}-x^t\|^2_{B^TB} )\\
	\leq &\frac{L}{2}\|x^{t+1}-x^t\|^2+\frac{L}{2}\|x^t-x^{t-1}\|^2+ \frac{\beta}{2} (\|x^t-x^{t-1}\|^2_{B^TB} +\|Ax^t-b\|^2)\\
	&-\frac{\beta}{2} ( \|(x^t-x^{t-1})-(x^{t+1}-x^t)\|^2_{B^TB}+\|A(x^{t+1}-x^t)\|^2 )-\langle \epsilon_t-\epsilon_{t-1}, x^{t+1}-x^t \rangle \\
	\end{split}
	\end{equation}
\end{lemma}
\begin{proof}
	Using the optimality condition of $x^{t+1}$ and $x^t$ in the update rule in \eqref{alg:general1}, we obtain
	
	\begin{equation}
	\langle g(x^t)+A^T\mu^t+\beta A^T(Ax^{t+1}-b)+\beta B^TB (x^{t+1}-x^t), x^{t+1}-x\rangle\leq 0, \forall x
	\end{equation}
	and 
	\begin{equation}
	\langle g(x^{t-1})+A^T\mu^{t-1}+\beta A^T(Ax^{t}-b)+\beta B^TB (x^{t}-x^{t-1}), x^{t}-x\rangle\leq 0, \forall x
	\end{equation}
	
	Replacing $g(x^t)$ by $\nabla f(x^t)+\epsilon_t$ and $g(x^{t-1})$ by $\nabla f(x^{t-1})+\epsilon_{t-1}$, and using the update rule \eqref{alg:gnenral2}

	\begin{equation}
	\langle \nabla f(x^t)+\epsilon_t+A^T \mu^{t+1}+\beta B^TB (x^{t+1}-x^t), x^{t+1}-x\rangle\leq 0, \forall x
	\end{equation}
	\begin{equation}
	\langle \nabla f(x^{t-1})+\epsilon_{t-1}+A^T \mu^t+\beta B^TB (x^{t}-x^{t-1}), x^{t}-x\rangle\leq 0, \forall x
	\end{equation}

	Now choose $x=x^t$ in the first inequality and $x=x^{t+1}$ in the second one, adding two inequalities together, we obtain

	\begin{equation}
	\langle \nabla f(x^{t})-\nabla f(x^{t-1})+\epsilon_t-\epsilon_{t-1}+A^T (\mu^{t+1}-\mu^t)+\beta B^TB \big( (x^{t+1}-x^t )-(x^t-x^{t-1})  \big), x^{t+1}-x^t \rangle\leq 0
	\end{equation}
	
	Rearranging above terms, we have
	
	\begin{equation}\label{equ:lemma2}
	\begin{split}
	&\langle A^T (\mu^{t+1}-\mu^t), x^{t+1}-x^t \rangle\\
	&\leq -\langle \nabla f(x^{t})-\nabla f(x^{t-1})+\epsilon_t-\epsilon_{t-1}+\beta B^TB \big( (x^{t+1}-x^t)- (x^t-x^{t-1}) \big), x^{t+1}-x^t \rangle
	\end{split}
	\end{equation}
	
	We first re-express the lhs of above inequality.
	
	\begin{equation}
	\begin{split}
	&\langle A^T (\mu^{t+1}-\mu^t), x^{t+1}-x^t \rangle\\
	=& \langle \beta A^T (Ax^{t+1}-b),x^{t+1}-x^t \rangle\\
	=&\langle \beta (Ax^{t+1}-b), Ax^{t+1}-b-(Ax^t-b) \rangle\\
	=&\beta \|Ax^{t+1}-b\|^2-\beta \langle Ax^{t+1}-b, Ax^t-b\rangle\\
	=&\frac{\beta}{2} (\|Ax^{t+1}-b\|^2-\|Ax^t-b\|^2+\|A(x^{t+1}-x^t)\|^2 )
	\end{split}
	\end{equation} 
	
	Next, we bound the rhs of \eqref{equ:lemma2}.
	
	\begin{equation}
	\begin{split}
	&-\langle \nabla f(x^{t})-\nabla f(x^{t-1})+\epsilon_t-\epsilon_{t-1}+\beta B^TB \big( (x^{t+1}-x^t)- (x^t-x^{t-1}) \big), x^{t+1}-x^t \rangle\\
	=& -\langle  \nabla f(x^{t})-\nabla f(x^{t-1})+\epsilon_t-\epsilon_{t-1}  , x^{t+1}-x^t\rangle-\beta \langle B^TB\big( (x^{t+1}-x^t )-(x^t-x^{t-1})  \big), x^{t+1}-x^t  \rangle\\
	\overset{a}{\leq}& \frac{L}{2}\|x^{t+1}-x^t\|^2+\frac{1}{2L}\|\nabla f(x^t)-\nabla f(x^{t-1})\|^2-\langle \epsilon_t-\epsilon_{t-1}, x^{t+1}-x^t \rangle\\
	&- \beta \langle B^TB\big( (x^{t+1}-x^t )-(x^t-x^{t-1})  \big), x^{t+1}-x^t  \rangle\\
	\overset{b}{\leq} & \frac{L}{2}\|x^{t+1}-x^t\|^2+\frac{L}{2}\|x^t-x^{t-1}\|^2-\langle \epsilon_t-\epsilon_{t-1}, x^{t+1}-x^t \rangle\\
	&- \beta \langle B^TB\big( (x^{t+1}-x^t )-(x^t-x^{t-1})  \big), x^{t+1}-x^t  \rangle\\
	=& \frac{L}{2}\|x^{t+1}-x^t\|^2+\frac{L}{2}\|x^t-x^{t-1}\|^2-\langle \epsilon_t-\epsilon_{t-1}, x^{t+1}-x^t \rangle\\
	+&\frac{\beta}{2} \big( \|x^t-x^{t-1}\|^2_{B^TB}-\|x^{t+1}-x^t\|^2_{B^TB}-\|(x^t-x^{t-1})-(x^{t+1}-x^t)\|^2_{B^TB}    \big),
	\end{split}
	\end{equation}
	where the inequality (a) uses Cauchy-Schwartz inequality, (b) holds from the smoothness assumption on $f$.

	Combine all pieces together, we obtain

	\begin{equation}
	\begin{split}
	&\frac{\beta}{2} ( \|Ax^{t+1}-b\|^2+\|x^{t+1}-x^t\|^2_{B^TB} )\\
	\leq &\frac{L}{2}\|x^{t+1}-x^t\|^2+\frac{L}{2}\|x^t-x^{t-1}\|^2+ \frac{\beta}{2} (\|x^t-x^{t-1}\|^2_{B^TB} +\|Ax^t-b\|^2)\\
	&-\frac{\beta}{2} ( \|(x^t-x^{t-1})-(x^{t+1}-x^t)\|^2_{B^TB}+\|A(x^{t+1}-x^t)\|^2 )-\langle \epsilon_t-\epsilon_{t-1}, x^{t+1}-x^t \rangle \\
	\end{split}
	\end{equation}
\end{proof}
Same with \cite{hong2017prox}, we define the potential function 

\begin{equation}
P_{c,\beta} (x^{t+1},x^t,\mu^{t+1})=L_\beta(x^{t+1},\mu^{t+1})+\frac{c\beta}{2}( \|Ax^{t+1}-b\|^2+\|x^{t+1}-x^t\|^2_{B^TB} )
\end{equation}

\begin{lemma}\label{lemma:potential}
	If Assumption \ref{assumption:f} holds, we have following 
	
	\begin{equation}\label{equ:potential}
	\begin{split}
	&P_{c,\beta}(x^{t+1},x^t,\mu^{t+1})\\
	\leq &P_{c,\beta}(x^{t},x^{t-1},\mu^t) -(\frac{\beta}{2}-\frac{cL}{2}-\frac{2c+1}{2}  )\|x^{t+1}-x^t\|^2+(\frac{3L^2}{\beta \sigma_{\min}}+ \frac{cL}{2}) \|x^{t-1}-x^t\|^2\\
	&-(\frac{c\beta}{2}-\frac{3\beta\|B^TB\|}{\sigma_{\min}}) \|(x^{t+1}-x^t)-(x^t-x^{t-1})\|^2_{B^TB}+(\frac{c}{4}+\frac{\beta}{3}) \|\epsilon_{t-1}-\epsilon_t\|^2 +\frac{1}{2}\|\epsilon_t\|^2
	\end{split}
	\end{equation}
\end{lemma}

\begin{proof}
	\begin{equation}
	\begin{split}
	&P_{c,\beta} (x^{t+1},x^t,\mu^{t+1}) \\
	\leq &L_\beta (x^t,\mu^t)+\frac{c\beta}{2}( \|x^t-x^{t-1}\|_{B^TB}+\|Ax^t-b\|^2)- (\frac{\beta}{2}-\frac{cL}{2}) \|x^{t+1}-x^t\|^2\\
	+& (\frac{3L^2}{\beta \sigma_{\min}}+ \frac{cL}{2}) \|x^{t-1}-x^t\|^2-(\frac{c\beta}{2}-\frac{3\beta\|B^TB\|}{\sigma_{\min}}) \|(x^{t+1}-x^t)-(x^t-x^{t-1})\|^2_{B^TB}\\
	+&\frac{3}{\beta} \|\epsilon_t-\epsilon_{t-1}\|^2+\langle \epsilon_t,x^{t+1}-x^t\rangle-c\langle\epsilon_t-\epsilon_{t-1},x^{t+1}-x^t \rangle\\
	\leq& P_{c,\beta} (x^t,x^{t-1},\mu^t)- (\frac{\beta}{2}-\frac{cL}{2}) \|x^{t+1}-x^t\|^2+(\frac{3L^2}{\beta \sigma_{\min}}+ \frac{cL}{2}) \|x^{t-1}-x^t\|^2\\
	&-(\frac{c\beta}{2}-\frac{3\beta\|B^TB\|}{\sigma_{\min}}) \|(x^{t+1}-x^t)-(x^t-x^{t-1})\|^2_{B^TB}+\frac{c}{4}\|\epsilon_{t-1}-\epsilon_t\|^2\\
	&+c\|x^{t+1}-x^t\|^2+\frac{1}{2}\|\epsilon_t\|^2+\frac{1}{2}\|x^{t+1}-x^t\|^2+\frac{\beta}{3}\|\epsilon_t-\epsilon_{t-1}\|^2\\
	=& P_{c\beta}(x^t,x^{t-1},\mu^t)- (\frac{\beta}{2}-\frac{cL}{2}-\frac{2c+1}{2}  )\|x^{t+1}-x^t\|^2+(\frac{3L^2}{\beta \sigma_{\min}}+ \frac{cL}{2}) \|x^{t-1}-x^t\|^2\\
	&-(\frac{c\beta}{2}-\frac{3\beta\|B^TB\|}{\sigma_{\min}}) \|(x^{t+1}-x^t)-(x^t-x^{t-1})\|^2_{B^TB}+(\frac{c}{4}+\frac{\beta}{3}) \|\epsilon_{t-1}-\epsilon_t\|^2 +\frac{1}{2}\|\epsilon_t\|^2
	\end{split}
	\end{equation}
	
	where the second inequality holds from  the Cauchy-Schwartz inequality.
	
	We require that 
	
	$$ \frac{c\beta}{2}-\frac{3\beta \|B^TB\|}{\sigma_{\min}}\geq 0, $$ which is satisfied when 
	
	\begin{equation}\label{equ:requirement_c}
	c\geq \frac{6\|B^TB\|}{\sigma_{\min}}
	\end{equation}

	We further require 
	$$ (\frac{\beta}{2}-\frac{cL}{2}-\frac{2c+1}{2}  )\geq (\frac{3L^2}{\beta \sigma_{\min}}+\frac{cL}{2}),$$
	which will be used later in the telescoping.
	
	Thus we require  
	
	\begin{equation}\label{equ:requirement_beta}
	\beta\geq 2cL+2c+1+\frac{6L^2}{\beta\sigma_{\min}}.
	\end{equation}
	
	and choose $\beta\geq CL+\frac{2c+1}{2} + \frac{1}{2}\sqrt{(2cL+2c+1)^2+\frac{24L^2}{\sigma_{\min}}}$	
\end{proof}

Now we do summation over both side of \eqref{equ:potential} and have	

\begin{equation}
\begin{split}
&\sum_{t=1}^{T} [(\frac{\beta}{2}-\frac{cL}{2} -\frac{2c+1}{2}  )\|x^{t+1}-x^t\|^2-(\frac{3L^2}{\beta \sigma_{\min}}+ \frac{cL}{2}) \|x^{t-1}-x^t\|^2]\\
&+ \sum_{t=1}^{T}(\frac{c\beta}{2}-\frac{3\beta\|B^TB\|}{\sigma_{\min}}) \|(x^{t+1}-x^t)-(x^t-x^{t-1})\|^2_{B^TB}\\
\leq& P_{c\beta}(x^1,x^0,\mu^0)-P_{c\beta}(x^{T+1},x^T, \mu^T)+\sum_{t=1}^{T} [(\frac{c}{4}+\frac{\beta}{3})\|\epsilon_{t-1}-\epsilon_t\|^2+\frac{1}{2} \|\epsilon_t\|^2]
\end{split}
\end{equation}

rearrange terms of above inequality.

\begin{equation}
\begin{split}
&\sum_{t=1}^{T-1}(\frac{\beta}{2}-\frac{cl}{2}-\frac{2c+1}{2}-\frac{3L^2}{\beta\sigma_{\min}}-\frac{cl}{2}) \|x^{t+1}-x^t\|^2 +(\frac{\beta}{2}-\frac{cl}{2}-\frac{2c+1}{2})\|x^{T+1}-x^T\|^2\\
\leq & P_{c\beta}(x^1,x^0,\mu^0)-P_{c\beta}(x^{T+1},x^T, \mu^T)+ (\frac{3L^2}{\beta\sigma_{\min}}+\frac{cl}{2} )\|x^1-x^0\|^2+\sum_{t=1}^{T}[ (\frac{c}{4}+\frac{\beta}{3})\|\epsilon_{t-1}-\epsilon_t\|^2+\frac{1}{2} \|\epsilon_t\|^2]
\end{split}
\end{equation}

Next we show $P_{c\beta}$ is lower bounded

The following lemma is from Lemma 3.5 in \cite{hong2016decomposing}, we present here for completeness. 

\begin{lemma}
	Suppose Assumption \ref{assumption:f} are satisfied, and $(c,\beta)$ are chosen according to \eqref{equ:requirement_beta} and \eqref{equ:requirement_c}. Then the following state holds true
	
	$ \exists \underline{P}~~ s.t., ~~P_{c\beta} (x^{t+1},x^t, \mu^{t+1})\geq \underline{P}> -\infty $
\end{lemma}

\begin{proof}
	\begin{equation}
	\begin{split}
	L_{\beta} (x^{t+1}, \mu^{t+1})=&f(x^{t+1})+\langle \mu^{t+1}, Ax^{t+1}-b \rangle+\frac{\beta}{2}\|Ax^{t+1}-b\|^2\\
	=&f(x^{t+1})+\frac{1}{\beta}\langle \mu^{t+1},\mu^{t+1}-\mu^t\rangle+\frac{\beta}{2}\|Ax^{t+1}-b\|^2\\
	=& f(x^{t+1})+\frac{1}{2\beta} (\|\mu^{t+1}\|^2-\|\mu^t\|^2+\|\mu^{t+1}-\mu^{t}\|^2)+\frac{\beta}{2}\|Ax^{t+1}-b\|^2.
	\end{split}
	\end{equation}

	Sum over both side, we obtain

	\begin{equation}
	\sum_{t=1}^{T}L_{\beta}(x^{t+1},\mu^{t+1})=\sum_{t=1}^{T} \big(f(x^{t+1})+\frac{\beta}{2}\|Ax^{t+1}-b\|^2+\frac{1}{2\beta}\|\mu^{t+1}-\mu^t\|^2 \big)+\frac{1}{2\beta} (\|\mu^{T+1}\|^2-\|\mu^1\|^2)
	\end{equation}
	
	By assumption 2, above sum is lower bounded, which implies that the sum of the potential function is also lower bounded (Recall $P_{c,\beta} (x^{t+1},x^t,\mu^{t+1})=L_\beta(x^{t+1},\mu^{t+1})+\frac{c\beta}{2}( \|Ax^{t+1}-b\|^2+\|x^{t+1}-x^t\|^2_{B^TB} )$ ).  Thus we have 
	
	$$P_{c\beta} (x^{t+1}, x^t, \mu^{t+1})>-\infty, \forall t>0$$
\end{proof}

In the next step, we are ready to provide the convergence rate. Following \cite{hong2016decomposing}, we define the convergence criteria

\begin{equation}\label{equ:convergence_criteria}
Q(x^{t+1}, \mu^{t+1})=\|\nabla L_\beta(x^{t+1},\mu^t)\|^2+\|Ax^{t+1}-b\|^2
\end{equation}
It is easy to see, when $Q(x^{t+1},\mu^t)=0$, $\nabla f(x)+A^T\mu=0$ and $Ax=b$, which are KKT condition of the problem.

\begin{equation}
\begin{split}
&\|\nabla L_\beta (x^t,\mu^{t-1})\|^2\\
=& \|\nabla f(x^t)-\nabla f(x^{t-1})+\epsilon_t-\epsilon_{t-1}+A^T(\mu^{t+1}-\mu^t)+\beta B^TB (x^{t+1}-x^t )\|^2\\
\leq & 4L^2 \|x^{t}-x^{t-1}\|^2+4\|\mu^{t+1}-\mu^t\|^2\|A^TA\|+4\beta^2\|B^TB(x^{t+1}-x^t)\|^2+4\|\epsilon_t-\epsilon_{t-1}\|^2
\end{split}
\end{equation}

Using the proof in Lemma \ref{lemma:lemma_on_mu}, we know there exist two positive constants c1 c2 c3 c4

$$Q(x^t,\mu^{t-1}) \leq c_1\|x^t-x^{t+1}\|^2+c_2\|x^t-x^{t-1}\|^2+c_3\|B^TB \big((x^{t+1}-x^t)-(x^t-x^{t-1})    \big)\|^2+c_4\|\epsilon_t-\epsilon_{t-1}\|^2.$$

Using Lemma \ref{lemma:potential}, we know there must exist a constant $\kappa$ such that 

\begin{equation}
\begin{split}
&\sum_{t=1}^{T-1} Q(x^t,\mu^{t-1})\\
\leq &\kappa (P_{c\beta}(x^1,x^0,\mu^0)-P_{c\beta}(x^{T+1},x^T, \mu^T)+\sum_{t=1}^{T}[ (\frac{c}{4}+\frac{\beta}{3})\|\epsilon_{t-1}-\epsilon_t\|^2+\frac{1}{2} \|\epsilon_t\|^2]  )+c_4\sum_{t=1}^{T-1} \|\epsilon_t-\epsilon_{t-1}\|^2\\
\leq &  \kappa (P_{c\beta}(x^1,x^0,\mu^0)-\underline{P}+\sum_{t=1}^{T}[ (\frac{c}{4}+\frac{\beta}{3})\|\epsilon_{t-1}-\epsilon_t\|^2+\frac{1}{2} \|\epsilon_t\|^2]  )+c_4\sum_{t=1}^{T-1} \|\epsilon_t-\epsilon_{t-1}\|^2
\end{split}
\end{equation}	

Divide both side by $T$ and take expectation 

\begin{equation}\label{equ:bound_Q}
\begin{split}
\frac{1}{T} \mathbb{E}\sum_{t=1}^{T} Q(x^t,\mu^{t-1})
\leq \frac{1}{T} \kappa (P_{c\beta} (x_1,x^0,\mu^0)-\underline{P})+&\frac{\kappa}{T}[ \sum_{t=1}^{T} (\frac{c}{4}+\frac{\beta}{3}  )\mathbb{E}\|\epsilon_{t-1}-\epsilon_t\|^2+\frac{1}{2}\|\epsilon_t\|^2]\\
+&\frac{c_4}{T}\sum_{t=1}^{T-1}\mathbb{E}\|\epsilon_t-\epsilon_{t-1}\|^2  
\end{split}
\end{equation}

Now we bound the R.H.S. of above equation.

Recall we choose the mini-batch size $\sqrt{T}$, $\epsilon_t=\varepsilon_t+\tilde{\varepsilon}_t$ and $\varepsilon_t\leq c_1/\sqrt{T}$

\begin{equation}
\|\epsilon_{t-1}-\epsilon_t\|^2\leq 2\mathbb{E}(\|\epsilon_{t-1}\|^2+\|\epsilon_t\|^2)\leq 4 \mathbb{E} ( \|\varepsilon_t\|^2+\|\tilde{\varepsilon}_t\|^2+\|\varepsilon_{t-1}\|^2+\|\tilde{\varepsilon}_{t-1}\|^2 )\leq \frac{8c_1}{T}+\frac{8\sigma^2}{T} 
\end{equation}

Similarly we can bound $\|\epsilon_t\|^2$. Combine all pieces together, we obtain

$$ \frac{1}{T} \mathbb{E}\sum_{t=1}^{T} Q(x^t,\mu^{t-1}) \leq  \frac{1}{T} \kappa (P_{c\beta} (x_1,x^0,\mu^0)-\underline{P})+ \frac{1}{T}(\kappa c_5+c_6 \sigma^2),$$
where $c_5, c_6$ are some universal positive constants.

Notice $ \min_{t} \mathbb{E} Q(x^t,\mu^{t-1})\leq \frac{1}{T} \mathbb{E}\sum_{t=1}^{T} Q(x^t,\mu^{t-1}) $, we have $  \min_{t} \mathbb{E} Q(x^t,\mu^{t-1})\leq (C+\sigma^2)/T $ where $C$ is a universal positive constant.

\subsection{Convergence on the dual update}

If the dual objective function is non-convex, we just follow the exact analysis in our proof on the primal problem. Notice the analysis on the dual update is easier than primal one, since we do not have the error term $\tilde{\epsilon}_t$. Therefore, we have the algorithm converges to stationary solution  with rate $\mathcal{O} (1/T)$in criteria $Q$.

If the dual objective function is linear or convex, the update rule reduce to Extra \citep{hong2016decomposing,shi2015extra} the convergence result of stochastic setting  can be adapted from the proof in \citep{shi2015extra}. Since it is not the main contribution of this paper, we omit the proof here.

\end{document}